\PassOptionsToPackage{table}{xcolor}
\documentclass[10pt, a4paper, reqno]{amsart}

\usepackage[T1]{fontenc}
\usepackage[utf8]{inputenc}
\usepackage[russian,main=english]{babel}

\usepackage{amsmath,amsthm,amssymb}
\usepackage{amscd}
\numberwithin{equation}{section}
\usepackage{times}
\usepackage{bbm}
\usepackage{cite}
\usepackage{enumitem}
\usepackage[left=1.15in,right=1.15in,top=1.15in,bottom=1.15in]{geometry} 
\usepackage{calc}
\usepackage{ifthen,tabularx,graphicx,multirow}
\usepackage{rotating}
\usepackage{hyperref}
\usepackage{tikz}
\usepackage{csquotes}
\usepackage{mathrsfs}
\usepackage{dsfont}
\usepackage[utf8]{inputenc}
\usepackage{wrapfig}
\usepackage[many]{tcolorbox}

\usepackage{tikz-cd}
\usetikzlibrary{tikzmark}
\usetikzlibrary{babel}
\usepackage{ etoolbox}
\usepackage{tcolorbox}
\usepackage{enumitem}
\usepackage{scrextend}

\usepackage{algorithm2e}
\newcolumntype{Y}{>{\centering\arraybackslash}X}
\newtheorem{lemma}{Lemma}
\newtheorem{proposition}[lemma]{Proposition}
\newtheorem{theorem}[lemma]{Theorem}
\numberwithin{lemma}{section}

\theoremstyle{definition}
\newtheorem{definition}[lemma]{Definition}

\newcommand{\nh}{\text{NH}}

\newcommand{\strbdepsph}{\epsilon_{\mathbb{P}h\mathrm{B}}^{\mathrm{s}}(\mathrm{p})}
\newcommand{\strbdepsp}{\epsilon_{\mathbb{P}\mathrm{B}}^{\mathrm{s}}(\mathrm{p})}

\newcommand{\gprob}{\Gamma^{\mathrm{ran}}}

\newcommand{\pr}{\mathbb{P}}
\theoremstyle{definition}
\newtheorem{remark}[lemma]{Remark}

\theoremstyle{definition}

\theoremstyle{definition}
\newtheorem{example}[lemma]{Example}

\newcount\colveccount
\newcommand*\colvec[1]{
\global\colveccount#1
\begin{pmatrix}
	\colvecnext
}
\def\colvecnext#1{
	#1
	\global\advance\colveccount-1
	\ifnum\colveccount>0
	\\
	\expandafter\colvecnext
	\else
\end{pmatrix}
\fi
}
\makeatletter
\newcounter{framedeqn}
\AtBeginEnvironment{subequations}{\let\c@equation\c@framedeqn}
\makeatother

\newcommand{\nncf}{\mathcal{C}\mathcal{F}}

\newcommand{\real}{\mathbb{R}}

\newcommand{\sgn}{\text{sgn}}

\newcommand{\supp}{\text{supp}}

\newcommand{\disM}{\mathrm{dist}_{\mathcal{M}}}

\newcommand{\dyadic}{\mathbb{D}}

\usepackage{mathtools}
\DeclarePairedDelimiter\floor{\lfloor}{\rfloor}
\DeclarePairedDelimiter\ceil{\lceil}{\rceil}

\newcommand{\argmineps}{\underset {\varphi \in \mathcal{NN}_{\mathbf{N},L}}{\operatorname {argmin}_{\epsilon}}} 
\newcommand{\nncfeps}{\mathcal{CF}^{\epsilon,\hat \epsilon}_{r} }
\newcommand{\nnsubs}{^\mathcal{NN}_{f,r,\epsilon, \mathcal{R}, (\mathbf{N},L)}}

\begin{document}
\title[The mathematics of adversarial attacks in AI]{The mathematics of adversarial attacks in AI \\
	---\\  Why deep learning is unstable despite the existence of stable neural networks}

\author{A. Bastounis} 
\address{Department of Mathematics, King's College London}
\email{alexander.bastounis@kcl.ac.uk}

\author{A. C. Hansen} 
\address{DAMTP, University of Cambridge}
\email{ach70@cam.ac.uk}

\author{V. Vla\v{c}i\'{c}} 
\address{D-ITET, ETH Zürich}
\email{vlacicv@mins.ee.ethz.ch}


\begin{abstract}
	The unprecedented success of deep learning (DL) makes it unchallenged when it comes to classification problems. However, it is well established that the current DL methodology produces universally unstable neural networks (NNs). The instability problem has caused a substantial research effort -- with a vast literature on so-called adversarial attacks -- yet there has been no solution to the problem. Our paper addresses why there has been no solution to the problem, as we prove the following: any training procedure based on training ReLU neural networks for classification problems with a fixed architecture will yield neural networks that are either inaccurate or unstable (if accurate) -- despite the provable existence of both accurate and stable neural networks for the same classification problems. The key is that the stable and accurate neural networks must have variable dimensions depending on the input, in particular, variable dimensions is a necessary condition for stability. Our result points towards the paradox that accurate and stable neural networks exist, however, modern algorithms do not compute them. This yields the question: if the existence of neural networks with desirable properties can be proven, can one also find algorithms that compute them? There are cases in mathematics where provable existence implies computability, but will this be the case for neural networks? The contrary is true, as we demonstrate how neural networks can provably exist as approximate minimisers to standard optimisation problems with standard cost functions, however, no randomised algorithm can compute them with probability better than $1/2$. 
\end{abstract}
\maketitle

\setcounter{tocdepth}{1}
\tableofcontents

\section{Introduction}
Neural networks (NNs) \cite{devore_hanin_petrova_2021, pinkus1999approximation, HHdl2019} and deep learning (DL) \cite{LeCun_Nature} have seen incredible success, in particular in classification problems \cite{Nature_Cancer}. However, neural networks become universally unstable (non-robust) when trained to solve such problems in virtually any application \cite{Des2023, akhtar2018threat,carlini2018audio,huang2018some, finlayson2019adversarial, Heaven_Nature, AdcockHansenBook, SzZ-14, antun2020instabilities, Choi_IEEE, Choi2}, making the non-robustness issue one of the fundamental problems in artificial intelligence (AI). The vast literature on this issue -- often referring to the instability phenomenon as vulnerability to adversarial attacks -- has not been able to solve the problem. Thus, we are left with the key question:
\begin{displayquote}
	\normalsize
	{\it Why does deep learning yield universally unstable methods for classification? }
\end{displayquote}
\vspace{1mm}
\noindent In this paper we provide mathematical answers to this question in connection with Smale's 18th problem on the limits of AI. 

The above problem has become particularly relevant as the instability phenomenon yields non-human-like behaviour of AI with misclassifications by DL methods being caused by small perturbations that are so tiny that human sensor systems such as eyes and ears cannot detect the tiny change. The non-robustness issue has thus caused serious concerns among scientists \cite{finlayson2019adversarial, Heaven_Nature, AdcockHansenBook}, in particular in applications where trustworthiness of AI is a key feature. Moreover, the instability phenomenon has become a grave matter for policy makers for regulating AI in safety critical areas where trustworthiness is a must, as suggested by the European Commission's outline for a legal framework for AI: 

\vspace{1mm}
\enquote{{\it In the light of the recent advances in artificial intelligence (AI), the serious negative consequences of its use for EU citizens and organisations have led to multiple initiatives from the European Commission to set up the principles of a trustworthy and secure AI. Among the identified requirements, the concepts of robustness and explainability of AI systems have emerged as key elements for a future regulation of this technology.}}

 -- Europ. Comm. JCR Tech. Rep. (January 2020) \cite{hamonrobustness}. 
 \vspace{1mm}
  
\enquote{{\it On AI, trust is a must, not a nice to have. [...] The new AI regulation will make sure that Europeans can trust what AI has to offer. [...] High-risk AI systems will be subject to strict obligations before they can be put on the market: [requiring] High level of robustness, security and accuracy.}}

 --- 
Europ. Comm. outline for legal AI (April 2021) \cite{EU_PressRelease}. 

\vspace{1mm}
   The concern is also shared on the American continent, especially regarding security and military applications. Indeed, the US Department of Defence has spent millions of dollars through DARPA on project calls to cure the instability problem.  
 The strong regulatory emphasis on trustworthiness, stability (robustness), security and accuracy leads to potential serious consequences given that modern AI techniques are universally non-robust. Current state-of-the-art AI techniques may be illegal in certain key sectors given their fundamental lack of robustness. The lack of a cure for the instability phenomenon in modern AI suggests a methodological barrier applicable to current AI techniques, and hence should be viewed in connection with Smale's 18th problem on the limits of AI.  
 
\subsection{Main theorems -- Methodological barriers, Smale's 18th problem and the limits of AI}
Smale's 18th problem, from the list of mathematical problems for the 21st century \cite{21century_Smale}, echoes Turing's paper from 1950 \cite{Turing_1950} on the question of existence of AI. Turing asks if a computer can think, and suggests the imitation game (Turing test) as a test for his question about AI. Smale takes the question even further and asks in his 18th problem: \emph{what are the limits of AI?} The question is followed by a discussion on the problem that ends as follows.  {\it ``Learning is a part of human intelligent activity. The corresponding mathematics is suggested by the theory of repeated games, neural nets and genetic algorithms.''}

Our contribution to the program on Smale's 18th problem are the following limitations and methodological barriers on modern AI highlighted in (I) and (II) below. These results provide mathematical answers to the question on why there has been no solution to the instability problem.

\begin{itemize}[leftmargin=*]
\item[(I)] {\bf Theorem \ref{theorem:NNCV}}: There are basic methodological barriers in state-of-the-art DL based on ReLU NNs. Indeed, any training procedure based on training ReLU NNs for many simple classification problems with a fixed architecture will yield neural networks that are either inaccurate or unstable (if accurate) -- despite the provable existence of both accurate and stable neural networks for the same classification problems. Moreover, variable dimensions on NNs is necessary for stability for ReLU NNs.  
\end{itemize}

\noindent Theorem \ref{theorem:NNCV} points towards the paradox that accurate and stable neural networks exist, however, modern algorithms do not compute them. This yields the question: 
\vspace{1mm}
\begin{displayquote}
	\normalsize
	{\it If the existence of neural networks can be proven, can one also find algorithms that compute them? In particular, there are cases in mathematics where provable existence implies computability, but will this be the case for neural networks?}
\end{displayquote}
\vspace{1mm}

\noindent  We address this question even for provable existence of NNs in standard training scenarios.

\begin{itemize}[leftmargin=*]
\item[(II)] {\bf Theorem \ref{theorem:NNNonComp}}: There are NNs that provably exist as approximate minimisers to standard optimisation problems with standard cost functions, however, no randomised algorithm can compute them with probability better than $1/2$.
\end{itemize}

\noindent A detailed account of the results and the consequences can be found in \S \ref{sec:NNintro} and \S \ref{sec:NNComp}. 

\subsection{Phase transitions and generalised hardness of approximation (GHA)} 

Theorem \ref{theorem:NNNonComp} can be understood within the framework of \emph{generalized hardness of approximation} (GHA) \cite{opt_big, AdcockHansenBook, gazdag2022generalised, comp_stable_NN22, AIM, Johan1, CRP, Hansen2021, BCH1}, which describes a specific phase transition phenomenon. In many cases, it is straightforward to compute an $\epsilon$-approximation to a solution of a computational problem for $\epsilon > \epsilon_1 > 0$. However, when $\epsilon < \epsilon_1$ (the approximation threshold), a phase transition occurs, wherein it is suddenly difficult, or even infeasible, to obtain an $\epsilon$-approximation. This difficulty could manifest as non-computability or intractability (e.g. non-polynomial time complexity). GHA extends the concept of hardness of approximation \cite{Arora2007} from discrete computations to more general computational problems.

In particular, Theorem \ref{theorem:NNNonComp} establishes lower bounds on the approximation threshold $\epsilon_1 > 0$ for computing NNs in classification tasks. This theorem builds upon the initial work on GHA introduced in \cite{opt_big} for convex optimisation (see also Problem 5 (J. Lagarias) in \cite{AIM}) and further developed in \cite{comp_stable_NN22, gazdag2022generalised} for NNs in AI and inverse problems. The theory of GHA is part of the larger framework of the Solvability Complexity Index (SCI) hierarchy \cite{comp_stable_NN22, SCI, Hansen_JAMS, CRAS, Ben-Artzi_FoCM2021, colbrook2019foundations, Colbrook_2019, colbrook_spec_meas, Hansen2016ComplexityII}.

\section{Main Results I -- Trained NNs become unstable despite the existence\\ of stable and accurate NNs}\label{sec:NNintro}
In this section we will explain our contributions to understanding the instability phenomenon. We consider the simplest deep learning problem of approximating a given classification function 
\begin{equation}\label{eq:the_fs}
	f:  [0,1]^d \rightarrow \{0,1\},
\end{equation}
by constructing a neural network from training data. Let $\mathcal{NN}_{\mathbf{N},L}$ with $\mathbf{N} := (N_L=1,N_{L-1},\dotsc,N_1,N_0 = d)$ denote the set of all $L$-layer neural networks (with $L \geq 2$) under the ReLU nonlinearity with $N_\ell$ neurons in the $\ell$-th layer (see Section \ref{sec:NNBasics} for definitions and explanations of these concepts). We assume that the cost function $\mathcal{R}$ is an element of  
\begin{equation}\label{eq:CF}
	\mathcal{CF}_{r} = \{\mathcal{R}: \mathbb{R}^{r} \times \mathbb{R}^{r} \rightarrow \mathbb{R}_+ \cup \{\infty\} \, \vert \,  \mathcal{R}(v,w) = 0 \text{ iff } v = w\}.
\end{equation}

\begin{remark}[Choice of cost functions]
Note that the choice of class of cost functions defined in \eqref{eq:CF} will be used in Theorem \ref{theorem:NNCV} is to demonstrate how one can achieve great generalisability properties of the trained network. It is worth mentioning however that we show that expanding this class to include, for example, regularised cost functions will not cure the instability phenomenon (see Section \ref{sec:NNCVInterp} (II) for more detail).
\end{remark}
As we will discuss the stability of neural networks, we introduce the idea of \emph{well-separated and stable sets} to exclude pathological examples whereby the training and validation sets have elements that are arbitrarily close to each other in a way that could make the classification function jump subject to a small perturbation. Specifically, given a classification function $f: [0,1]^d \rightarrow \{0,1\}$, we define the family of well-separated and stable sets $\mathcal{S}^f_{\delta}$ with separation at least $2\delta$ according to
\begin{equation*}
	\begin{split}
		&\mathcal{S}^f_{\delta} = \Big\{\{x^1, \hdots, x^m\} \subset  [0,1]^d \, \vert \, m \in \mathbb{N}, \\ 
		& \qquad \qquad  \min_{x^i\neq x^j} \|x^i - x^j\|_{\infty} \geq 2 \delta, f(x^j+y) = f(x^j) \text{ for }\|y\|_{\infty} < \delta\text{ satisfying }x^j+y \in [0,1]^d \Big \}.
	\end{split}
\end{equation*}
We also set $r \vee s$ to be the maximum of $r$ and $s$ and $r \wedge s$ to be the minimum of $r$ and $s$. Finally, we use the notation $\mathcal{B}_{\epsilon}^{\infty}$ to refer to the open ball of radius $\epsilon$ in the $\ell^{\infty}$ norm. With this notation established, we are now ready to state our first main result.

\begin{theorem}[Instability of trained NNs despite existence of a stable NN]\label{theorem:NNCV}
	There is an uncountable collection $\mathcal{C}_1$ of classification functions $f$ as in \eqref{eq:the_fs} -- with fixed $d \geq 2$ -- and a constant $C>0$ such that the following holds. For every $f \in \mathcal{C}_1$, any norm $\|\cdot\|$ and every $\epsilon>0$, there is an uncountable family $\mathcal{C}_2$ of probability distributions on $[0,1]^{d}$ so that for any $\mathcal{D} \in \mathcal{C}_2$, any neural network dimensions $\mathbf{N} = (N_L=1,N_{L-1},\dotsc,N_1,N_0=d)$ with $L \geq 2$, any $\mathrm{p} \in (0,1)$, any positive integers $q$, $r$, $s$ with 
	\begin{equation}\label{eq:r-NN-lower-bound}
		r+s \geq C \, \max\big\{\,\mathrm{p}^{-3}  ,\, q^{3/2} \big[(N_1+1) \dotsb (N_{L-1}+1)\big]^{3/2} \big\},
	\end{equation}
	any training data $\mathcal{T} = \{x^1, \hdots, x^r\}$ and validation data $\mathcal{V} = \{y^1, \hdots, y^s\}$, where the $x^j$ and $y^j$ are drawn independently at random from $\mathcal{D}$, the following happens with probability exceeding $1-\mathrm{p}$.
	
	\begin{itemize}
		\item[(i)]
		 \underline{\emph{(Success -- great generalisability)}}. We have $\mathcal{T}, \mathcal{V} \in \mathcal{S}^f_{\varepsilon((r\vee s)/\mathrm{p})}$, where $\varepsilon(n)=(Cn)^{-4}$,
		and, for every $\mathcal{R} \in \mathcal{CF}_r$, there exists a ${\phi}$ such that
		\begin{equation}\label{eq:NNOptimisationProblem}
			{\phi} \in  \mathop{\mathrm{arg min}}_{\varphi \in \mathcal{NN}_{\mathbf{N},L}}  \mathcal{R} \big(\{\varphi({x}^j)\}_{j=1}^r,\{f({x}^j)\}_{j=1}^r\big)
		\end{equation}
		and
		\begin{equation}\label{eq:NNCVCorrectClassification}
			\phi(x) = f(x) \quad \forall x \in \mathcal{T} \cup \mathcal{V}.
		\end{equation}
		\item[(ii)]
		 \underline{\emph{(Any successful NN in $\mathcal{NN}_{\mathbf{N},L}$ -- regardless of architecture -- becomes universally unstable)}}. Yet, for any $\hat{\phi} \in \mathcal{NN}_{\mathbf{N},L}$ (and thus, in particular, for $\hat\phi=\phi$) and any monotonic $g: \real \to \real$, there is a subset $\mathcal{\tilde T}\subset \mathcal{T} \cup \mathcal{V}$ of the combined training and validation set of size $|\mathcal{\tilde T}| \geq q$, such that there exist uncountably many universal adversarial perturbations $\eta \in \mathbb{R}^d$ so that for each $x \in \mathcal{\tilde T}$ we have 
		\begin{equation}\label{eq:NNCVIncorrectClassification}
			|g \circ \hat\phi(x+\eta) -  f(x+\eta)| \geq 1/2,\quad \|\eta\| < \epsilon, \quad |\supp(\eta)| \leq 2.
		\end{equation}
		\item[(iii)]  \underline{\emph{(Other stable and accurate NNs exist)}}.
		However, there exists a stable and accurate neural network $\psi$ that satisfies $\psi(x) = f(x)$ for all $x \in \mathcal{B}_{\epsilon}^{\infty}(\mathcal{T} \cup \mathcal{V})$, when $\epsilon \leq \varepsilon((r\vee s)/\mathrm{p})$. 
	\end{itemize}
\end{theorem}
We remark in passing that the training and validation data $\mathcal{T}$ and $\mathcal{V}$ in Theorem \ref{theorem:NNCV} are technically not sets, but randomised multisets, as some of the samples $x^j$ or $y^j$ may be repeated.

\begin{remark}[The role of $g$ in (ii) in Theorem \ref{theorem:NNCV}] The purpose of the monotone function $g: \real \to \real$ in (ii) in Theorem \ref{theorem:NNCV} is to make the theorem as general as possible. In particular, a popular way of creating an approximation to $f$ is to have a network combined with a thresholding function $g$. This would potentially increase the approximation power compared to only having a neural network, however, Theorem \ref{theorem:NNCV} shows that adding such a function does not cure the instability problem. 
\end{remark}

\begin{figure}
	\fbox{\parbox{.8\textwidth}{
			\centering
			\includegraphics[scale=0.31]{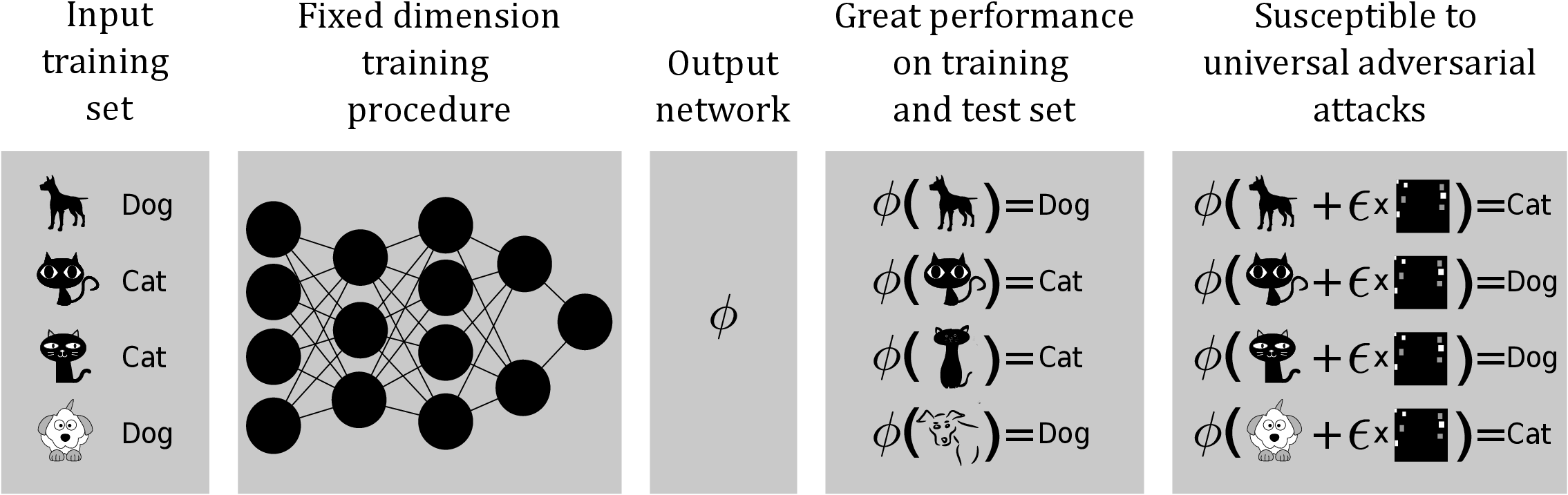}\\\vspace{2mm}\hrule\vspace{2mm}
			\includegraphics[scale=0.31]{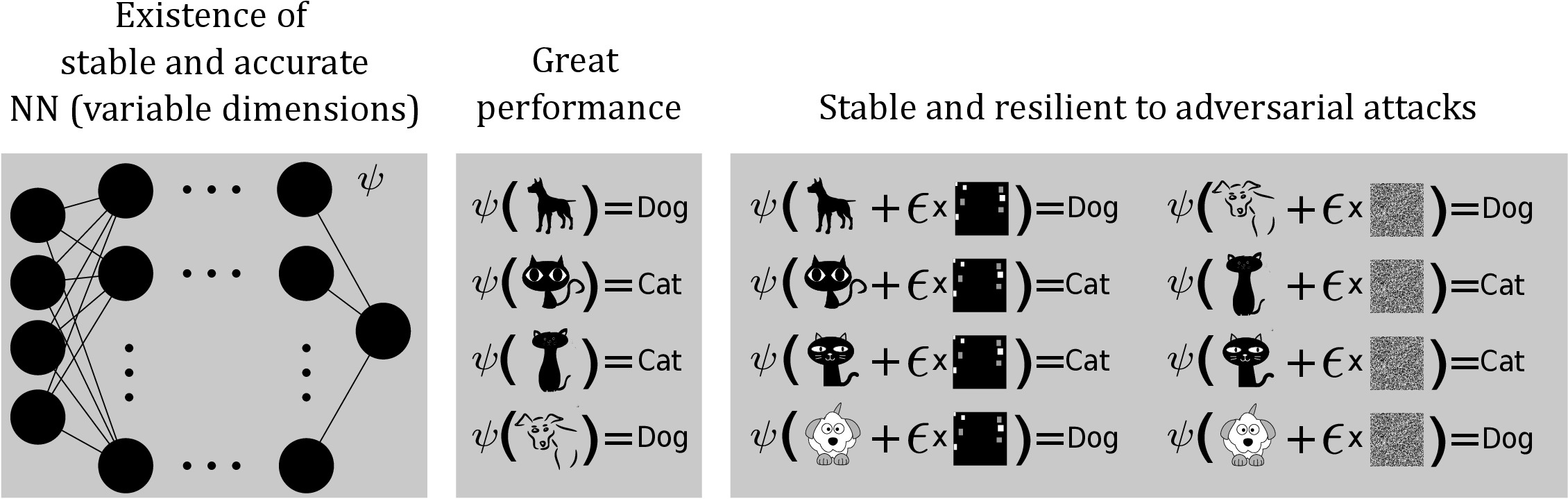}
		}
	
}
	
	\caption{({\bf Training with fixed architecture yields instability -- Variable dimensions on NNs is necessary for stability for ReLu NNs}). A visual interpretation of Theorem \ref{theorem:NNCV}. A fixed dimension training procedure can lead to excellent performance and yet be highly susceptible to adversarial attacks, even if there exists a NN which has both great performance and excellent stability properties. However, such a stable and accurate ReLu network must have variable dimensions depending on the input.}
\end{figure}

\subsection{Interpreting Theorem \ref{theorem:NNCV}}\label{sec:NNCVInterp}
In this section we discuss in detail the implications of Theorem \ref{theorem:NNCV} with regard to Smale's 18th problem. First, note that Theorem \ref{theorem:NNCV} demonstrates a methodological barrier applicable to current DL approaches. This does not imply that the instability problem in classification cannot be resolved but it does imply that in order to overcome these instability issues one will have to change the methodology. Second, Theorem \ref{theorem:NNCV} provides guidance on which methodologies will not solve the instability issues. 
In order to make the exposition easy to read, we will now summarise in non-technical terms what Theorem \ref{theorem:NNCV} says. 
\begin{itemize}[leftmargin=*]
	\item[(I)] \underline{\emph{Performance comes at a cost -- Accurate DL methods inevitably become unstable}}. Theorem \ref{theorem:NNCV}  shows that there are basic classification functions and distributions where standard DL methodology yields trained NNs with great success in terms of generalisability and performance -- note that the size of the validation set $\mathcal{V}$ in Theorem \ref{theorem:NNCV} can become arbitrary large. However, \eqref{eq:r-NN-lower-bound} demonstrates how greater success (better generalisability) implies more instabilities. Indeed, the NNs -- regardless of architecture and training procedure -- become either successful and unstable, or unsuccessful.  
	\item[(II)] \underline{\emph{There is no remedy within the standard DL framework}}. Note that (ii) in Theorem \ref{theorem:NNCV} demonstrates that there is no remedy within the standard DL framework to cure the instability issue described in Theorem \ref{theorem:NNCV}. The reason why is that standard DL methods will fix the architecture (i.e. the class $\mathcal{N}$ of NNs that one minimises over) of the neural networks.  Indeed, the misclassification in \eqref{eq:NNCVIncorrectClassification} happens for any neural network $\hat{\phi} \in \mathcal{NN}_{\mathbf{N},L}$.  This means that, for example, zero loss training \cite{Donoho_PNAS}, or any attempt using adversarial training \cite{madry2018towards, Goodfellow} -- i.e. computing   
	\begin{equation*}
		\label{opt_GANS}
		\min_{\phi\in\mathcal{N}} E_{x \sim \mathcal{D}} \max_{z \in \mathcal{U}} \mathcal{L}(\phi(x+z), f(x)),
	\end{equation*}
	where $\mathcal{N} \subset \mathcal{NN}_{\mathbf{N},L}$ is any collection of NNs described by a specific architecture, $\mathcal{U} \subset \mathbb{R}^d$ and $\mathcal{L}$ is any real valued cost function -- will not solve the problem.  In fact, (ii) in Theorem \ref{theorem:NNCV} immediately implies that adversarial training will reduce the performance if it increases the stability.

	\item[(III)]  \underline{\emph{There are accurate and stable NNs, but DL methods do not find them}}. Note that (iii) in Theorem \ref{theorem:NNCV} demonstrates that there are stable and accurate NNs for many classification problems where DL methods produce 
	unstable NNs. Thus, the problem is not that stable and accurate NNs do not exist; instead, the problem is that DL methods do not find them. The reason is that the dimensions and architectures of the stable and accurate networks will change depending on the size and properties of the data set. 
	
	\item[(IV)]  \underline{\emph{Why instability? -- Unstable correlating features are picked up by the trained NN}}. 
	In addition to the statement of Theorem \ref{theorem:NNCV}, the proof techniques illustrate the root causes of the problem. The reason why one achieves the great success described by (i) in Theorem \ref{theorem:NNCV} is that the successful NN picks up a feature in the training set that correlates well with the classification function but is itself unstable. This phenomenon is empirically investigated in \cite{Madry_BugsNIPS2019}. 
	
	\item[(V)]  \underline{\emph{No training model where the dimensions of the NNs are fixed can cure instability}}. 
	Note that (ii) in Theorem \ref{theorem:NNCV} describes the reason for the failure of the DL methodology to produce a stable and accurate NN. Indeed, as pointed out above, the dimensions of the stable and accurate NN will necessary change with the amount of data in the training and validation set. 
	
\item[(VI)]  \underline{\emph{Adding more training data cannot cure the general instability problem}}. 
	Note that \eqref{eq:r-NN-lower-bound} in Theorem \ref{theorem:NNCV} shows that adding more training data will not help. In fact, it can make the problem worse. Indeed, \eqref{eq:r-NN-lower-bound} allows $s$ -- the number of elements in the test set - to be set so that $s = 1$, and $r$ -- the number of training data -- can be arbitrary large. Hence, if $r$ becomes large and $s$ is small, 	then the trained NN -- if successful -- will (because of (ii)) start getting instabilities on the training data. In particular, the network has seen the data, but will misclassify elements arbitrary close to seen data.
	
\item[(VII)]  \underline{\emph{Comparison with the No-Free-Lunch Theorem}}. The celebrated No-Free-Lunch Theorem has many forms. However, the classical impossibility result we refer to (Theorem 5.1 in \cite{MachLearnCUP2014}) states that for any learning algorithm for classification problems there exists a classification function $f$ and a distribution $\mathcal{D}$ that makes the algorithm fail. Our Theorem \ref{theorem:NNCV} is very different in the way that it is about instability. Moreover, it is an impossibility result specific for deep learning. Thus, the statements are much stronger.

 Indeed -- in contrast to the single classification function $f$ and distribution $\mathcal{D}$ making a fixed algorithm fail in the No-Free-Lunch Theorem -- Theorem \ref{theorem:NNCV} shows the existence of uncountably many classification functions $f$ and distributions $\mathcal{D}$ such that for any fixed architecture DL will either yield unstable and successful NNs or unsuccessful NNs. This happens despite the existence of stable and accurate NNs for exactly the same problem. Moreover, Theorem \ref{theorem:NNCV} shows how NNs can generalise well given relatively few training data compared to the test data, but at the cost of non-robustness (note that this is in contrast to the No-Free-Lunch theorem wherein few training samples leads to a lack of generalisation). See also \cite{gottschling2020troublesome} for other `No free lunch' theorems.
\end{itemize}

\section{Main Results II -- NNs may provably exist, but no algorithm can compute them}\label{sec:NNComp}

The much celebrated Universal Approximation Theorem is widely used as an explanation for the wide success of NNs in the sciences and in AI, as it guarantees the existence of neural networks that can approximate arbitrary continuous functions. In essence, any task that can be handled by a continuous function can also be handled by a neural network. 

\begin{theorem}[Universal Approximation Theorem \cite{pinkus1999approximation}]\label{thrm:UniversalApproximationTheorem}
Suppose that $\sigma \in C(\mathbb{R})$, where $C(\mathbb{R})$ denotes the set of continuous functions on $\mathbb{R}$. Then the set of neural networks with non-linearity $\sigma$ is dense in $C(\mathbb{R}^d)$ in the topology of uniform convergence on compact sets, if and only if $\sigma$ is not a polynomial. 
\end{theorem}

Theorem \ref{theorem:NNCV} illustrates basic methodological barriers in DL and suggests the following fundamental question: 
\begin{displayquote}
	\normalsize
	{\it If we can prove that a stable and well generalisable neural network exists, why do algorithms fail to compute them? }
\end{displayquote}
This question is not only relevant because of Theorem \ref{theorem:NNCV} but also the Universal Approximation Theorem that demonstrates -- in theory -- that there are very few limitations on what NNs can do. Yet, there is clearly a barrier that the desirable NNs that one can prove exist -- as shown in Theorem \ref{theorem:NNCV} and in many cases follow from the Universal Approximation Theorem -- are not captured by standard algorithms.

\subsection{The weakness of the universal approximation theorem -- When will existence imply computability?}

\noindent The connection between the \emph{provable existence} of a mathematical object and its \emph{computability} (that there is an algorithm that can compute it) touches on the foundations of mathematics \cite{smith_2013}. Indeed, there are cases in mathematics -- with the ZFC\footnote{Zermelo-Fraenkel axiomatic system with the axiom of choice, which is the standard axiomatic system for modern mathematics} axioms  -- where the fact that one can prove mathematically a statement about the existence of the object will imply that one can find an algorithm that will compute the object when it exists. Consider the following example: 

\vspace{2mm}

\begin{example}[{\bf When provable existence implies computability}]\label{ex:Dio}
Consider the following basic computational problem concerning Diophantine equations:
\begin{displayquote}
	\normalsize 
	{\it Let $\Theta$ be a collection of polynomials in $\mathbb{Z}[x_1,x_2,\dotsc,x_n]$ with integer coefficients, where $n \in \mathbb{N}$ can be arbitrary.  Given a polynomial $p \in \Theta$, does there exist an integer vector $a \in \mathbb{Z}^n$ such that $p(a) = 0$, and if so, compute such an $a$.}
\end{displayquote}
Note that in this case we have that "being able to prove $\Rightarrow$ being able to compute" as the following implication holds for the ZFC  model \cite{poonen_2014}:
\begin{equation}\label{eq:phase}
	\begin{tikzcd}[column sep=0mm, row sep=0.9cm]
		\fbox{\begin{tabular}{c}
				For any polynomial $p \in \Theta$, one can prove -- given the ZFC axioms -- that there exists \\
				an integer vector $a \in \mathbb{Z}^n$ such that $p(a) = 0$, or a negation of this statement.
		\end{tabular} }
		\arrow[d,Rightarrow,shift left, " "{name=U, below,sloped}]
		\\
		\fbox{ \begin{tabular}{c}
				There exists an algorithm $\Gamma$ taking any polynomial $p \in \Theta$ such that $\Gamma(p) = \text{`no'} $ if \\ there is no $a \in \mathbb{Z}^n$ such that $p(a) = 0$, otherwise $\Gamma(p) = a$ where $a \in \mathbb{Z}^n$ such that \\ $p(a)=0$.
		\end{tabular} }
	\end{tikzcd}
\end{equation}
The above implication is true, subject to ZFC being consistent and that theorems in ZFC about integers are true \cite{poonen_2014}. In particular, being able to prove existence or not of integer valued zeroes of polynomials in $\Theta$ implies the existence of an algorithm that can compute integer valued zeroes of polynomials in $\Theta$ and determine if no integer valued zero exists. 
\end{example}

There is a substantial weakness with the Universal Approximation Theorem and the vast literature on approximation properties of NNs, in that they provide little insight into how NNs should be computed, or indeed if they can be computed. As Example \ref{ex:Dio} suggests, there are cases where provable existence implies computability. Hence, we are left with the following basic problem:
\begin{displayquote}
	\normalsize 
	{\it If neural networks can be proven to exist, will there exist algorithms that can compute them? If this is not the case in general, what about neural networks that can be proven to be approximate minimisers of basic cost functions?}
\end{displayquote}
As we see in the next sections, the answer to the above question is rather delicate.  

\begin{remark}
		Although the universal approximation theorem does not directly apply to non-constant classification functions in the class \eqref{eq:the_fs}, if we consider a classification function restricted to a finite set (e.g. training and validation sets), it will have a continuous extension and hence the universal approximation theorem will apply. Furthermore, recent results discuss existence theorems in the general setting of \eqref{eq:the_fs} \cite{David1}.
	\end{remark}
\subsection{Inexactness and floating point arithmetic}
The standard model for computation in most modern software is floating point arithmetic. This means that even a rational number like $1/3$ will be approximated by a base-2 approximation. Moreover, the floating point operations yield errors, that -- in certain cases -- can be analysed through backward error analysis, which typically show how the computed solution in floating point arithmetic is equivalent to a correct computation with an approximated input.  Hence, in order to provide a realistic analysis we use the model of computation with inexact input as emphasised by S. Smale in his list of mathematical problems for the 21st century:
\begin{displayquote}
	\normalsize
	{\it ``But real number computations and algorithms which work only in exact arithmetic can offer only limited understanding. Models which process approximate inputs and which permit round-off computations are called for.''}
	\\[5pt]
	\rightline{ --- S. Smale (from the list of mathematical problems for the 21st century \cite{21century_Smale}) \hspace{15mm}}
\end{displayquote}
To model this situation, we shall assume that an algorithm designed to compute a neural network is allowed to see the training set to an arbitrary accuracy decided at runtime. More precisely, for a given training set $\mathcal{T} = \{x^1,x^2,\dotsc,x^r\}$, we assume that the algorithm (a Turing \cite{Turing_Machine} or Blum-Shub-Smale (BSS) \cite{BSS_machine} machine) is equipped with an oracle $\mathscr{O}$ that can acquire the true input to any accuracy $\epsilon$.  Specifically, the algorithm cannot access the vectors $x^1,x^2,\dotsc,x^r$ but rather, for any $k \in \mathbb{N}$, it can call the oracle $\mathscr{O}$ to obtain $x^{1,k},x^{2,k},\dotsc,x^{r,k}$ such that 
\begin{equation}\label{eq:oracle22}
	\|x^{i,k} - x^{i}\|_{\infty} \leq 2^{-k}, \qquad \text{ for } i=1,2,\dotsc,r \text{ and }  \, \forall k\in\mathbb{N},
\end{equation} 
see \S \ref{sec:inexact} for details. 

Another key assumption when discussing the success of the algorithm is that it must be ``oracle agnostic'', i.e., it must work with any choice of the oracle $\mathscr{O}$ satisfying \eqref{eq:oracle22}. In the Turing model the Turing machine accesses the oracle via an oracle tape and in the BSS model the BSS machine accesses the oracle through an oracle node. This extended computational model of having inexact input is standard and can be found in many areas of the mathematical literature - we mention only a small subset here: E. Bishop \cite{bishop1967foundations}, F. Cucker \& S. Smale \cite{Cucker_Smale97}, C. Fefferman \& B. Klartag \cite{Fefferman_Klartag, Fefferman_Klartag2}, K. Ko \cite{Ko1991ComplexityTO} and L. Lov\'{a}sz \cite{lovasz1987algorithmic}.

\subsection{Being able to prove existence may imply being able to compute -- but not in Deep Learning} We now examine the difference between being able to prove the existence of a neural network and the ability to compute it, even in the case when the neural network is an approximate minimiser. 
Recall the typical training scenario of neural networks in \eqref{eq:NNOptimisationProblem} where one tries to find
\begin{equation*}
	{\phi} \in  \mathop{\mathrm{arg min}}_{\varphi \in \mathcal{NN}_{\mathbf{N},L}}  \mathcal{R} \left(\{\varphi({x}^j)\}_{j=1}^r,\{f({x}^j)\}_{j=1}^r\right),
\end{equation*}
where $f$ is the decision function, $\mathcal{R}$ is the cost function, and $\mathcal{T} = \{x^j\}_{j=1}^r$ is the training set. 
However, one will typically not reach the actual minimiser, but rather an approximation. Hence, we define \emph{the approximate argmin}.
\begin{definition}[The approximate argmin]
	Given an $\epsilon > 0$, an arbitrary set 
	$X$, a totally ordered set 
	$Y$, and a function $g\colon X\to Y$, the approximate
	$
	{\displaystyle \operatorname {argmin}_{\epsilon}}
	$ 
	over some subset $S \subset X$ is defined by
	\begin{equation}\label{eq:argminEpsDef}
		{\underset {x\in S}{\operatorname {argmin}_{\epsilon}}}\,g(x):=\{x \in S \, \vert \, g(x) \leq g(y) + \epsilon \,\, \forall \, y \in X\}
	\end{equation}
\end{definition}
To accompany the idea of the approximate argmin, we will also consider cost functions that are bounded with respect to the $\ell^{\infty}$ norm:
\begin{equation}\label{eq:CFDEps}
	\mathcal{CF}^{\epsilon,\hat \epsilon}_{r} = \{\mathcal{R} \in \mathcal{CF}_r : \mathcal{R}(v,w) \leq \epsilon \implies \|v-w\|_{\infty} \leq \hat \epsilon \}.
\end{equation}
The computational problem we now consider is to compute a neural network that is an approximate minimiser and evaluate it on the training set (this is the simplest task that we should be able to compute):
\begin{equation}\label{eq:compute_phi}
	\phi(x^j), \qquad  {\phi} \in  \mathop{\mathrm{arg min}_{\epsilon}}_{\varphi \in \mathcal{NN}_{\mathbf{N},L}}  \mathcal{R} \left(\{\varphi({x}^j)\}_{j=1}^r,\{f({x}^j)\}_{j=1}^r\right), \quad \epsilon > 0, \quad j =1, \hdots, r.
\end{equation}
Hence, an algorithm $\Gamma$ trying to compute \eqref{eq:compute_phi} takes the training set $\mathcal{T}$ as an input (or to be precise, it calls oracles providing approximations to the $x^j$s to any precision, see see \S \ref{sec:inexact} for details) and outputs a vector in $\mathbb{R}^r$. Hopefully, $\|\Gamma(\mathcal{T}) - \{\phi(x^j)\}_{j=1}^r\|$ is sufficiently small. 

The next theorem shows that even if one can prove the existence of neural networks that are approximate minimisers to optimisation problems with standard cost functions, one may not be able to compute them.

\begin{theorem}[NNs may provably exist, but no algorithm can compute them]\label{theorem:NNNonComp}
	There is an uncountable collection $\mathcal{C}_1$ of classification functions $f$ as in \eqref{eq:the_fs} -- with fixed $d \geq 2$ -- such that the following holds. For
	\begin{enumerate}
		\item any neural  network dimensions $\mathbf{N} = (N_L=1,N_{L-1},\dotsc, \allowbreak N_1,N_0=d)$ with $L \geq 2$,
		\item any  $r \geq 3(N_1+1) \dotsb (N_{L-1}+1)$,
		\item any $\epsilon > 0$, $\hat \epsilon \in (0,1/2)$ and cost function $\mathcal{R} \in \nncf^{\epsilon,\hat\epsilon}_r$,
		\item any randomised algorithm $\Gamma$,
		\item any $\mathrm{p} \in [0,1/2)$,
	\end{enumerate}
	there is an uncountable collection $\mathcal{C}_2$ of training sets $\mathcal{T} = \{x^1,x^2,\dotsc,x^r\} \in \mathcal{S}^f_{\varepsilon'(r)}$ such that for each $\mathcal{T} \in \mathcal{C}_2$ there exists a neural network $\phi$, where 
	\[
	\phi \in \argmineps  \mathcal{R} \left(\{\varphi({x}^j)\}_{j=1}^r,\{f({x}^j)\}_{j=1}^r\right), 
	\]	
	however, the algorithm $\Gamma$ applied to the input $\mathcal{T}$ will fail to compute any such $\phi$ in the following way:
	\[
	\mathbb{P}\Big(\|\Gamma(\mathcal{T}) - \{\phi(x^j)\}_{j=1}^r\|_{*} \geq 1/4-3\hat\epsilon/4\Big ) > p ,
	\]
	where $* = 1,2 \text{ or } \infty$.
\end{theorem}

\subsection{A missing theory in AI -- Which NNs can be computed?}

If provable existence results about NNs were to imply that they could be computed by algorithms, the research effort to secure stable and accurate AI would -- in most cases -- be about finding the right algorithms that we know exist, due to the many neural network existence results \cite{devore_hanin_petrova_2021, pinkus1999approximation}. In particular, the key limitation for providing stable and accurate AI via deep learning -- at least in theory -- would be the capability of the research community.  However, as Theorem \ref{theorem:NNNonComp} reveals, the simplest existence results of NNs as approximate minimisers do not imply that they can be computed. Therefore the research effort moving forward must be about which NNs that can be computed by algorithms and how. Indeed, the limitations of deep learning as an instrument in AI will be determined by the limitations of existence of algorithms and their efficiency for computing NNs.

\begin{remark}[Theorem \ref{theorem:NNNonComp} is independent of the exact computational model]
	Note that the result above is independent of whether the underlying computational device is a BSS machine or a Turing machine. To achieve this, we work with a definition of an algorithm termed a \emph{general algorithm}. The corresponding definitions as well as a formal statement of Theorem \ref{theorem:NNNonComp} are detailed in \S \ref{Sec:SCIBackground} and Proposition  \ref{prop:NNNonComp-formal} respectively.
\end{remark}

\begin{remark}[Irrelevance of local minima]\label{rem:LocalMinima}
	Note that Theorem \ref{theorem:NNNonComp} has nothing to do with the potential issue of the optimisation problem having several local minima. Indeed, the general algorithms used in Theorem \ref{theorem:NNNonComp} are more powerful than any Turing machine or BSS machine as will be discussed further in Remark \ref{rem:GAPower}. \end{remark}

\begin{remark}[Hilbert's 10th problem]
Finally, we mention in passing that Theorem \ref{theorem:NNNonComp} demonstrates -- in contrast to Hilbert's 10th problem \cite{matiisevich1993hilbert} -- that non-computability results in DL do not prevent provable existence results. Indeed, because of the implication in \eqref{eq:phase} and the non-computability of Hilbert's 10th problem \cite{matiisevich1993hilbert, poonen_2014} (when $\Theta$ is the collection of all polynomials with integer coefficients in Example \ref{ex:Dio}), there are infinitely many Diophantine equations for which one cannot prove existence of an integer solution -- or a negation of the statement. 
\end{remark}

\section{Connection to previous work}

The literature documenting the instability phenomenon in DL is so vast that we can only cite a tiny subset here 
\cite{akhtar2018threat,carlini2018audio,huang2018some, finlayson2019adversarial, Heaven_Nature, AdcockHansenBook, SzZ-14, antun2020instabilities, Deep_Fool1, Deep_Fool2, Deep_Fool3, Goodfellow, shafahi2018adv, tyukin2020adversarial, Fawzi_NIPS2018}, see the references in the survey paper \cite{akhtar2018threat} for a more comprehensive collection. Below we will highlight some of the most important connections to our work.

\begin{itemize}[leftmargin=*]
	\item[(i)] {\bf \emph{Universality of instabilities in AI}}. A key feature of Theorem \ref{theorem:NNCV} is that it demonstrates how the perturbations are universal, meaning that one adversarial perturbation works for all the cases where the instability occurs -- as opposed to a specific input dependent adversarial perturbation. The Deep Fool program \cite{Deep_Fool1, Deep_Fool2, Deep_Fool3} -- created by S. Moosavi-Dezfooli, A. Fawzi, O. Fawzi, and P. Frossard -- was the first to establish empirically that adversarial perturbations can be made universal, and this phenomenon is also universal across different methods and architectures. For recent and related developments, see D. Higham and I. Tyukin et al. \cite{THGW21, tyukin2020adversarial} which describe instabilities generated by perturbations to the structure of a neural network, \cite{BGHHPSTZ,sutton2023adversarial} wherein instability to randomised perturbations are considered, as well as the results by L. Bungert and G. Trillos et al. \cite{bungert2023trilos_binary}, and S. Wang, N. Si, and J. Blanchet \cite{wang2023blanchet1}.
	
	\item[(ii)] {\bf \emph{Approximation theory and numerical analysis}}. There is a vast literature proving existence results of ReLU networks for deep learning, investigating their approximation power, where the recent work of R. DeVore, B. Hanin and G. Petrova \cite{devore_hanin_petrova_2021} also provides a comprehensive account of the contemporary developments. 
	The huge approximation literature on existence results and approximation properties of NNs prior to the year 2000 is well summarised by A. Pinkus in \cite{pinkus1999approximation}. 
	Our results suggest a program combining recent approximation theory \cite{Nina1, devore_hanin_petrova_2021} results with foundations of mathematics and numerical analysis to characterise the NNs that can be computed by algorithms. This aligns with the work of B. Adcock and N. Dexter \cite{adcock2020gap} that demonstrates the gap between what algorithms compute and the theoretical existence of NNs in function approximation with deep NNs. Note that results on existence of algorithms in learning -- with performance and stability guarantees -- do exist (see the work of P. Niyogi, S. Smale and S. Weinberger \cite{Smale_Weinberger}), but so far not in DL.

	\item[(iii)] {\bf \emph{Mathematical explanation of instability and impossibility results}}. Our paper is very much related to the work of 
	H. Owhadi, C. Scovel and T. Sullivan \cite{Owhadi_2015, Owhadi_SIREV} who "observe that learning and robustness are antagonistic properties". 
	The recent work of I. Tyukin, D. Higham and A. Gorban \cite{tyukin2020adversarial} and A. Shafahi, R. Huang, C. Studer, S. Feizi and T. Goldstein \cite{shafahi2018adv} demonstrate how the instability phenomenon increases with dimension showing universal lower bound on stability as a function of the dimension of the domain of the classification function. Note, however, that the results in this paper are independent of dimensions. The work of A. Fawzi, H. Fawzi and O. Fawzi \cite{Fawzi_NIPS2018} is also related, however, their results are about adversarial perturbations for any function, which is somewhat different from the problem that DL is unstable to perturbations that humans do not perceive. In particular, our results focus on how DL becomes unstable despite the fact that there is another device (in our case another NN) or a human that can be both accurate and stable.

	\item[(iv)] {\bf \emph{Proof techniques -- The Solvability Complexity Index (SCI) hierarchy}}.
	Initiated in \cite{Hansen_JAMS}, the mathematics behind the Solvability Complexity Index (SCI) hierarchy provides a variety of techniques to show lower bounds and impossibility results for algorithms -- in a variety of mathematical fields -- that provide the foundations for the proof techniques in this paper, see the works by V. Antun, J. Ben-Artzi, M. Colbrook, A. C. Hansen, M. Marletta, O. Nevanlinna, F. R\"osler,  M. Seidel. \cite{comp_stable_NN22, SCI, Hansen_JAMS, CRAS, Ben-Artzi_FoCM2021}. This is strongly related to the work of S. Weinberger \cite{Weinberger} on the existence of algorithms for computational problems in topology. The authors of this paper have also extended the SCI framework \cite{opt_big} in connection with the extended Smale's 9th problem.

	\item[(v)] {\bf \emph{Robust optimisation}}. Robust optimisation  \cite{Nemirovski_robust, Nemirovski_robust2, NemirovskiLMCO}, pioneered by A. Ben-Tal, L.  El Ghaoui and A. Nemirovski, is an essential part of optimisation theory addressing sensitivity to perturbations and inexact data in optimisation problems. There are crucial links to our results -- indeed, a key issue is that the instability phenomenon in DL leads to non-robust optimisation problems. In fact, there is a fundamental relationship between Theorem \ref{theorem:NNCV}, Theorem \ref{theorem:NNNonComp} and robust optimisation. Theorem \ref{theorem:NNNonComp} yields impossibility results in optimisation, where non-robustness is a key element. The big question is whether stable and accurate NNs -- with variable dimensions -- that exist as a result of Theorem \ref{theorem:NNCV} can be shown to be approximate minimisers of robust optimisation problems. This leads to the final question, would such problems be computable and have efficient algorithms?
	The results in this paper can be viewed as an instance of where robust optimisation meets the SCI hierarchy. This was also the case in the recent results on the extended Smale's 9th problem \cite{opt_big}. 
	
\end{itemize}


\section{Proofs of the main results}
\subsection{Some well known definitions and ideas from deep learning}\label{sec:NNBasics}
In this section we outline some basic well known definitions and explain the notation that will be useful for this paper. Many of these definitions can be found in \cite{Goodfellow-et-al-2016}.
For a vector $x \in \real^{N_1}$, we denote $x_i$ by the $i$th coordinate. Similarly, for a matrix $A \in \real^{N_1 \times N_2}$ for some dimensions $N_1 \in \mathbb{N}$ and $N_2 \in \mathbb{N}$ we denote $A_{i,j}$ by the entry of $A$ contained on the $i$th row and the $j$th column. 

Recall that for natural numbers $n_1,n_2$, an affine map $W: \real^{n_1} \to \real^{n_2}$ is a map such that there exists $A \in \real^{n_2 \times n_1}$ and $b \in \real^{n_2}$ so that for all $x \in \real^{n_1}$, $Wx = Ax + b$.
Let $L,d$ be natural numbers and let $\mathbf{N} := (N_L=1,N_{L-1},\dotsc,N_1,N_0)$ be a vector in $\mathbb{N}^{L+1}$ with $N_0 = d$. An \emph{neural network with dimensions} $(\mathbf{N},L)$ is a map $\phi: \real^d \to \real$ such that \[\phi = W^L \sigma W^{L-1} \sigma W^{L-2}\dotsc \sigma W^1\] where, for $l=1,2,\dotsc,L$, the map $W^l$ is an affine map from $\real^{N_{l-1}} \to \real^{N_{l}}$ i.e. $W^l x^l = A^l x^l + b^l$ where $b^l \in \real^{N_l}$ and $A^l \in \real^{N_{l} \times N_{l-1}}$. The map $\sigma: \real \to \real$ is interpreted as a coordinate-wise map and is called the \emph{non-linearity} or \emph{activation function}: typically, $\sigma$ is chosen to be continuous and non-polynomial \cite{pinkus1999approximation}.

In this paper, we focus on the well-known \emph{ReLU} non-linearity, which we denote by $\rho$. More specifically, for $x \in \real$ , we define $\rho(x)$ by $\rho(x) = 0$ if $x < 0$ and $\rho(x) = x$ if $x \geq 0$. We denote all neural networks with dimensions $(\mathbf{N},L)$ and the ReLU non-linearity by $\mathcal{NN}_{\mathbf{N},L}$. This will be the central object for our arguments.

\begin{remark}
	Although we focus on the ReLU non-linearity, it is possible to use the techniques presented in this paper to prove similar results for other non-linearities like the leaky ReLU \cite{LeakyReLU} $\rho^{\text{leaky}}$ and the parameterized ReLU \cite{ParReLU} $\rho^{param}_{\alpha}$ where
	\begin{equation*}
		\rho^{\text{leaky}}(x) = \begin{cases} 0.01 \cdot x & \text{if } x < 0\\
		x & \text { if } x \geq 0
		\end{cases}, \quad \rho^{param}_{\alpha}(x) = \begin{cases} \alpha x & \text{if } x < 0\\
		x & \text { if } x \geq 0
	\end{cases}
	\end{equation*}
\end{remark}

In this paper the most common norms we use are the $\ell^p$ norms: for a vector $x \in \real^d$ for some natural number $d$ and some $p \in [1,\infty)$, the $\ell^p$ norm of $x$ (which we denote by $\|x\|_p$) is given by $\|x\|_p = (\sum_{i=1}^d |x_i|^p)^{1/p}$. We also define the $\ell^{\infty}$ norm of $x$ (which we denote by $\|x\|_{\infty}$) by $\|x\|_{\infty}:= \max_{i = 1,2,\dotsc,d} |x_i|$. It is easy to see the following inequality: $\|x\|_{\infty} \leq \|x\|_2 \leq \|x\|_1$. 
We will denote the ball of radius $\epsilon$ about $x$ in the infinity norm by $\mathcal{B}^{\infty}_{\epsilon}(x)$ i.e. 
$\mathcal{B}^{\infty}_{\epsilon}(x) = \{y \in \real^d \, \vert \, \|y-x\|_{\infty} \leq \epsilon\}$. For a set $S$ we denote $\mathcal{B}^{\infty}_{\epsilon}(S)$ by $\cup_{x \in S} \mathcal{B}^{\infty}_{\epsilon}(x)$.

The cost function of a neural network is used in the training procedure: typically, one attempts to compute solutions to \eqref{eq:NNOptimisationProblem} where the function $\mathcal{R}$ is known as the cost function. In optimisation theory the cost function is sometimes known as the objective function and sometimes the loss function. 
Some standard choices for $\mathcal{R}$ include the following:
\begin{enumerate}
	\item \emph{Cross entropy cost function}, where $\mathcal{R}$ is defined by 
	\begin{equation*}
		\mathcal{R}(\{v^j\}_{j=1}^{r},\{w^j\}_{j=1}^{r}):= -\frac{1}{r}\sum_{j=1}^{r} \left(w^j \log(v^j) + (1-w^j) \log(1-v^j)\right)
	\end{equation*}
	The cross entropy function is only defined if $v^{j} \in [0,1]$: it is easy to extend this definition to $\mathcal{R}(\{v^j\}_{j=1}^{r},\{w^j\}_{j=1}^{r}):=\infty$
	when $v^j \notin [0,1]$ for some $j$.
\item \emph{Mean square error}, where $\mathcal{R}$ is defined by 
\begin{equation*}
	\mathcal{R}(\{v^j\}_{j=1}^{r},\{w^j\}_{j=1}^{r}):= \frac{1}{r}\|\{w^j\}_{j=1}^{r} - \{v^j\}_{j=1}^r\|_2^2
\end{equation*}
\item \emph{Root mean square error}, where $\mathcal{R}$ is defined by 
\begin{equation*}
	\mathcal{R}(\{v^j\}_{j=1}^{r},\{w^j\}_{j=1}^{r}):= \frac{1}{r}\|\{w^j\}_{j=1}^{r} - \{v^j\}_{j=1}^r\|_2
\end{equation*}
\item \emph{Mean absolute error}, where 
\begin{equation*} \mathcal{R}(\{v^j\}_{j=1}^{r},\{w^j\}_{j=1}^{r}):= \frac{1}{r}\|\{w^j\}_{j=1}^{r} - \{v^j\}_{j=1}^r\|_1 \end{equation*}
\end{enumerate}
Note that each of these functions are in $\mathcal{CF}_{r}$ where $\mathcal{CF}_{r}$ is defined in \eqref{eq:CF}.
\subsection{Lemmas and definitions common to the proofs of both Theorem \ref{theorem:NNCV} and \ref{theorem:NNNonComp}}

For both theorems the proof relies on the points $x^{k,\delta}\in\real^{N_0}$, defined for $k\in\mathbb{N}$, $\delta\geq 0$, $\kappa\in[1/4,3/4]$, and $a\in[1/2,1]$ as follows:
\begin{equation}\label{eq:theNNexes}
	x^{k,\delta}=\begin{cases}
		(a(k+1-\kappa)^{-1},0,\dotsc,0),&\text{if $k$ is odd} \\
		(a(k+1-\kappa)^{-1},\delta,0,0,\dotsc,0), &\text{if $k$ is even}
	\end{cases}.
\end{equation}
Both theorems also rely on some classification functions $f_a$ for $a \in [1/2,1]$, defined as follows: we set $f_{a}: \mathbb{R}^{N_0} \rightarrow \{0,1\}$
\begin{equation} \label{eq:TheFs}
	f_a(x) = \begin{cases} 1 & \text{ if } \lceil a/x_1\rceil \text{ is an odd integer }\\
		0 & \text{ otherwise }  \text{(including } x = 0\text) \end{cases}
\end{equation}
In particular, note that for any $\delta \geq 0$, $f_a(x^{k,\delta}) = 1$ if $k$ is even and $f_a(x^{k,\delta}) = 0$ if $k$ is odd.
The following three lemmas will be useful in both proofs. The first of these lemmas shows that finite collections of $x^{k,\delta}$ are well separated. Precisely, we will prove the following:
\begin{lemma}\label{lem:NN-separation-lemma}
	Let $a \in [1/2,1]$, $\kappa\in [1/4,3/4]$ and $\delta\geq 0$, and consider the points $x^{k,\delta}$ as given in \eqref{eq:theNNexes} and $f_a$ given as in \eqref{eq:TheFs}.
	Then, for every $K \in \mathbb{N}$, we have $ \{x^{1,\delta}, \dots, x^{K,\delta}\}\in \mathcal{S}^{f_a}_{\varepsilon'(K)}$, where  $\varepsilon'(n) := [(4n+3)(4n+4)]^{-1}$.
\end{lemma}
The purpose of the next lemma is to show that if $\delta > 0$ there is a neural network that matches $f_a$ on the $x^{k,\delta}$:
\begin{lemma}\label{lem:NNUnstableExists}
	Let $d$ be a natural number with $d\geq 2$, let $a \in [1/2,1]$, $\kappa\in [1/4,3/4]$ and $\delta > 0$, and consider the points $x^{k,\delta}$ as given in \eqref{eq:theNNexes} and $f_a$ given as in \eqref{eq:TheFs}. Fix neural networks dimensions $\mathbf{N} = (N_L=1,N_{L-1},\dotsc,N_1,N_0=d)$ with $L \geq 2$.
	Then there exists a neural network  $\tilde \varphi \in \mathcal{NN}_{\mathbf{N},L}$ with $\tilde\varphi(x^{k,\delta}) = f_a(x^{k,\delta})$ for all $k \in \mathbb{N}$.
\end{lemma}
Finally, the next lemma will be used to give examples of sets of vectors $\mathcal{W}$ and functions $f$ for which neural networks with fixed dimensions cannot exactly match $f$ on $\mathcal{W}$. More precisely, we shall show the following: 
\begin{lemma}\label{lem:NoNN}
	Let $d,t,m,L,N_1,N_2,\dotsc,N_L$ each be natural numbers and let $\mathcal{W}$ be a set of vectors with $\mathcal{W}= \{w^{1},w^{2},\dotsc,w^{t}\}\subset \real^{d}$. Suppose that each of the following apply
	\begin{enumerate}[label = (\arabic*)]
		\item $t \geq 3m\cdot(N_1 + 1)(N_2+1)\dotsb(N_L+1)$.
		\item $w^{1}_1 > w^{2}_1 > w^{3}_1 > \dotsb > w^{t}_1$ and $w^{1}_j = w^{2}_j = \dotsb = w^{t}_j = 0$ for $j=2,\dotsc,d$.
		\item $f: \real^d \to \{0,1\}$ is such that $f(w^i) \neq f(w^{i+1})$ for $i=1,2,\dotsc,t-1$. \label{item:differingf}
	\end{enumerate}
	Then for any neural network $\varphi \in \mathcal{NN}_{\mathbf{N},L}$ and any monotonic function $g: \real \to \real$ there exists a set $\mathcal{U} \subset \mathcal{W}$ such that $|\mathcal{U}| \geq m$ and $|g(\varphi(w)) - f(w)| \geq 1/2$ for all $w \in \mathcal{U}$.
\end{lemma}
The remainder of this subsection will be concerned with proving Lemma \ref{lem:NN-separation-lemma} to Lemma \ref{lem:NoNN}.
\subsubsection{Proof of Lemma \ref{lem:NN-separation-lemma}}
\begin{proof}[Proof of Lemma \ref{lem:NN-separation-lemma}]
	We must verify that $\min_{1 \leq i < j \leq K} \|x^{i,\delta}-x^{j,\delta} \|_{\infty} \geq 2\varepsilon'(K)$ and that for $k \leq K$ and vectors $y \in \real^{N_0}$ with $\|y\|_{\infty} < \varepsilon'(r)$ we have $f_a(x^{k,\delta}+y) = f_a(x^{k,\delta}).$
	
	For the first part, note that for distinct $i,j$ with $i,j\leq K$ we have 
	\begin{equation}\label{eq:NNSeparation}
		\|x^{i,\delta} - x^{j,\delta}\|_{\infty}\geq \left|\frac{a}{i+1 - \kappa} - \frac{a}{j+1 - \kappa}\right| =  \frac{|a(j-i)|}{(i+1 - \kappa)(j+1-\kappa)}\geq \frac{1}{2(K+1 - \kappa)(K-\kappa)}
	\end{equation}
	since $a|j-i| \geq a \geq 1/2$ and the condition that $i,j \leq K$ with at least one bounded by $K-1$ implies that $(i+1-\kappa)^{-1}(j+1-\kappa)^{-1} \geq (K+1 - \kappa)^{-1}(K- \kappa)^{-1}$. Since $\kappa \geq 1/4$, we get
	$\|x^{i,\delta} - x^{j,\delta}\|_{\infty}\geq \left[2(K+1 - 1/4)(K-1/4)\right]^{-1} \geq 2\varepsilon'(K)$. 
	
	Next, we let $k\leq K$ and $y \in \real^{N_0}$ be such that $\|y\|_{\infty} \leq \varepsilon'(K)$. We will establish that $f_a(x^{k,\delta}+y) = f_a(x^{k,\delta})$. Since $k\leq K$ and $\kappa \in [1/4,3/4]$, we have
	\begin{equation*}
		\frac{a(1-\kappa)}{(k+1-\kappa)k}>\frac{1}{(4K+3)(2K+2)}\geq y_1 \geq \frac{-1}{(4K+3)(2K+2)} \geq \frac{-a\kappa}{(k+1-\kappa)(k+1)}.
	\end{equation*} 
	We claim that this implies $a(x^{k,\delta}_1 + y_1)^{-1} \in (k,k+1]$. For the upper bound, note that
	\begin{equation*}
		\frac{y_1}{a} \geq  \frac{-\kappa}{(k+1-\kappa)(k+1)} =  \frac{1}{k+1} - \frac{1}{k+1-\kappa} = \frac{1}{k+1} - \frac{x^{k,\delta}_1}{a}.
	\end{equation*}
	Similarly, for the lower bound, we have
	\begin{equation*}
		\frac{y_1}{a} <  \frac{1-\kappa}{k(k+1-\kappa)} =  k^{-1}\left(\frac{k+1-\kappa}{k+1-\kappa} -\frac{k}{k+1-\kappa}\right)   = \frac{1}{k} - \frac{x^{k,\delta}_1}{a}.
	\end{equation*}
	Therefore $\lceil a/(x^{k,\delta}_1+y_1) \rceil = k+1$. Thus, for all $\|y\|_{\infty} < \varepsilon'(K)$, we have $f_a(x^{k,\delta}+y) = f_a(x^{k,\delta}) = 1$ for even $k$ and $f_a(x^{k,\delta}+y) = f_a(x^{k,\delta}) = 0$ for odd $k$, therefore establishing$ \{x^{1,\delta}, \dots, x^{K,\delta}\}\in \mathcal{S}^{f_a}_{\varepsilon'(K)}$.
\end{proof}
\subsubsection{Proof of Lemma \ref{lem:NNUnstableExists}}
\begin{proof}[Proof of Lemma \ref{lem:NNUnstableExists}]
	We set \[\tilde\varphi = W^L \rho W^{L-1} \rho W^{L-2}\dotsc \rho W^1\] where $W^\ell x = A^\ell x + b^\ell$ and $A^\ell \in \real^{N_{\ell} \times N_{\ell-1}}$, $b^\ell \in \real^{N_\ell}$ are defined as follows: let $A^1_{1,1} = 0$, $A^1_{1,2} = \delta^{-1}$ and $A^1_{i,j} = 0$ otherwise, and, for $\ell > 1$, $A^{\ell}_{1,1} = 1$ and $A^{\ell}_{i,j} = 0$ otherwise, and $b^\ell = 0$ for every $\ell$. Clearly
	\begin{equation*}
		W^1 x^{k,\delta} =\begin{cases} e_1 \in \mathbb{R}^{N_1}
			&\text{ if } k \text{ is even}\\
			\mathbf{0} \in \mathbb{R}^{N_1}& \text{ if } k \text{ is odd}\\
		\end{cases}
	\end{equation*}
	and it is therefore easy to see that $\tilde \varphi(x^{k,\delta}) = 1$ if $k$ is even and $\tilde \varphi(x^{k,\delta}) = 0$ if $k$ is odd. By the definition of $x^{k,\delta}$, we have $f_{a}(x^{k,\delta}) = 1$ if $k$ is even and $f_a(x^{k,\delta}) = 0$ if $k$ is odd, and therefore $\tilde \varphi(x^{k,\delta}) = f_a(x^{k,\delta})$ for all $k$. 
\end{proof}
\subsubsection{Proof of Lemma \ref{lem:NoNN}}
To prove Lemma \ref{lem:NoNN} we will state and prove the following:
\begin{lemma}\label{lem:NNReduce}
	Fix $m, n \in \mathbb{N}$, $A \in \real^{N \times N_0}$, $B \in \real^{m \times N}$, and $z \in \real^{N}$. Suppose that
	\[
	R = \lbrace \alpha^q,\alpha^{q+1},\alpha^{q+2},\dotsc, \alpha^{q + r-1}\rbrace\subset \real^{N_0}
	\]
	is a set such that $|R| \geq N+1$, the sequence $\{\alpha^k_1\}_{k=q}^{q+r-1}$ is strictly decreasing and  $\alpha^k_j=0$ for $j > 1$ and all $k$. Then there exist a matrix $C \in \real^{m \times N_0}$, a vector $v \in \real^m$, and a set $\mathcal{S} \subseteq R$ of the form $\mathcal{S} = \lbrace \alpha^s,\alpha^{s+1},\dotsc, \alpha^{s+t-1}\rbrace$ such that $|\mathcal{S}|\geq |R|/(N+1)$ and $B \rho (A \alpha + z) = C\alpha + v$, for all $\alpha \in \mathcal{S}$.
\end{lemma}

\begin{proof}[Proof of Lemma \ref{lem:NNReduce}]
	Write $B = (b_{j,k})_{j=1,k=1}^{j=m,k=N}$, $A = (a_{j,k})_{j=1,k=1}^{j=N,k=N_0}$. We claim that the set $\mathcal{Q}$ defined by
	\begin{equation*}\label{claim:numSigns}
		\mathcal{Q} = \lbrace \left(\sgn(a_{1,1}u_1 + w_{1}),\sgn(a_{2,1}u_1 + w_{2}),\dotsc ,\sgn(a_{N,1}u_1 + w_N)\right) \, \vert \, u \in R \rbrace.
	\end{equation*}
	contains at most $N+1$ (unique) elements, i.e., $|\mathcal{Q}|\leq N+1$, where we define $\sgn(x)=1$ for $x\geq 0$ and $\sgn(x)=-1$ for $x<0$.
	To see this, note that if we allow the value of $\beta$ to vary over $\real$, then each of the lines $y = a_{1,1}\beta + z_{1}$ , $y = a_{2,1}\beta + z_{2}$, \dots , $y = a_{N,1}\beta + z_N$ intersect the line $y = 0$ at most once. Between each of these intersections the vector $(\sgn(a_{1,1}\beta + z_{1}),\sgn(a_{2,1}\beta + z_{2}),\dotsc ,\sgn(a_{N,1}\beta + z_N))$ is constant. As there are at most $N$ such intersections, we note that if 
	\begin{equation*}
		\mathcal{Q}':=\lbrace \left(\sgn(a_{1,1}\beta + w_{1}),\sgn(a_{2,1}\beta + w_{2}),\dotsc ,\sgn(a_{N,1}\beta + w_N)\right) \, \vert \, \beta \in \mathbb{R} \rbrace.
	\end{equation*}
	then $|\mathcal{Q}'| \leq N+1$ follows because partitioning a line by at most $N$ intersections gives at most $N+1$ regions between the intersections. As $\mathcal{Q} \subseteq \mathcal{Q'}$
	the proof that $|\mathcal{Q}| \leq N+1$ is complete. 
	
	We can now define $\mathcal{S}$. By the pigeonhole principle and the fact that $|\mathcal{Q}|\leq N+1$, there exists a subset of $R$ with cardinality at least $|R|/(N+1)$ such that the vector \[\sgn(a_{\,\cdot,1}\, \alpha_1 + z) = \left(\sgn(a_{1,1}\alpha_1 + z_{1}),\sgn(a_{2,1}\alpha_1 + z_{2}),\dotsc ,\sgn(a_{N,1}\alpha_1 + z_N)\right)\] is constant over $\alpha$ in this subset. Let $\mathcal{S}$ be a subset of $R$ of maximal cardinality satisfying this constant sign condition. Then clearly $|\mathcal{S}| \geq |R|/(N+1)$.
	To see that $\mathcal{S} =  \lbrace \alpha^s,\alpha^{s+1},\dotsc, \alpha^{s+t-1}\rbrace$, for some $s$ and $t$, suppose by way of contradiction that no such $s$ and $t$ exist. Then there are $j_1$ and $k_1$ such that $j_1 + 1 < k_1$, $\alpha^{j_1}, \alpha^{k_1} \in \mathcal{S}$ and $\alpha^{j_1+1} \notin \mathcal{S}$. But then, as $\mathcal{S}$ is assumed to be of maximal cardinality, there must be an $\ell$ for which $\sgn(a_{\ell,1}\alpha^{j_1}_1 + z_1) = \sgn(a_{\ell,1}\alpha^{k_1}_1+z_1) \neq \sgn(a_{\ell,1}\alpha^{j_1+1}_1 + z_1)$. However, since $\{\alpha^j_1\}_{j=j_1}^{k_1}$ is a strictly decreasing sequence by assumption, we see that if $a_{\ell,1} \geq 0$ then $a_{\ell,1}\alpha^{j_1}_1 + z_1 \geq a_{\ell,1}\alpha^{j_1+1}_1+z_1 \geq a_{\ell,1}\alpha^{k_1}_1+z_1$ and similarly if $a_{\ell,1} < 0$ then $a_{\ell,1}\alpha^{j_1}_1 + z_1 < a_{\ell,1}\alpha^{j_1+1}_1+z_1 < a_{\ell,1}\alpha^{k_1}_1+z_1$ which is a contradiction. This establishes that $\mathcal{S} =  \lbrace \alpha^s,\alpha^{s+1},\dotsc, \alpha^{s+t-1}\rbrace$, for some $s$ and $t$. 
	
	We now show how to construct $C$ and $v$. Recall that, for all $\alpha\in \mathcal{S}$,  $\alpha_2 = \alpha_3 = \dotsb = \alpha_{N_0} = 0$, and so the $i$-th row of $B \rho(A\alpha+z)$ is given by $\sum_{j=1}^{N} b_{i,j} \rho(a_{j,1}\alpha_1 + z_{j})$. Since $\sgn(a_{j,1}\alpha_1 + z_j)$ is constant over $\alpha \in \mathcal{S}$, we must have that for each $j$ either
	$\rho(a_{j,1}\alpha_1 + z_{j}) = 0$ or $\rho(a_{j,1}\alpha_1 + w_{j}) = a_{j,1}\alpha_1 + z_j$, for all $\alpha\in \mathcal{S}$. In the former case, we define $d_{i,j} =0 $ and $y_{i,j} = 0$ and in the latter case we define $d_{i,j} = b_{i,j} a_{j,1}$ and $y_{i,j} = b_{i,j}z_{j}$. Therefore, by construction, the $i$-th row of
	$B \rho (A\alpha + z)$  is given by $\sum_{j=1}^{N} \left(d_{i,j} \alpha_1 + y_{i,j}\right)$. Thus, defining the matrix $C = (c_{i,j})_{i=1,j=1}^{i=m,j=N_0}$ and the vector $v\in\real^m$ according to
	\begin{equation*}
		c_{i,1} = \sum_{k=1}^{N} d_{i,k}, \quad c_{i,j} =   0, \text{ for }  j > 1, \qquad\text{and} \qquad  v_i = \sum_{k=1}^{N} y_{i,k}
	\end{equation*}
	immediately yields that the $i$-th row of $B\rho(A\alpha + z)$ satisfies $\sum_{k=1}^{N} \left(d_{i,k} \alpha_1 + y_{i,k}\right) = \sum_{k=1}^N c_{i,k}\alpha_k \,+ v_i$. As $i$ and $\alpha\in \mathcal{S}$ were arbitrary, this implies that $B \rho (A\alpha + z) = C\alpha + v$ for all $\alpha \in \mathcal{S}$, thereby concluding the proof of the lemma.
\end{proof}

With Lemma \ref{lem:NNReduce}, we can now prove Lemma \ref{lem:NoNN}.
\begin{proof}[Proof of Lemma \ref{lem:NoNN}]
	We begin by proving the following claim: 
	
	{\bf Claim: } There exists a set
	\[
	\mathcal{S} = \{w^s,w^{s+1},w^{s+2},\dotsc, w^{s+n}\} \subset \{w^1,w^2,\dotsc,w^t\}
	\]
	for some $s \in \mathbb{N}$ and $n \in \mathbb{N}$, a matrix
	$
	M \in \mathbb{R}^{1 \times N_0} 
	$	
	and a $z \in \mathbb{R}$ such that, for all $w \in \mathcal{S}$, we have 
	$
	\varphi(w) = Mw + z
	$ and so that $|\mathcal{S}|\geq 3m$.
	
	To see the validity of this claim, we proceed inductively by showing that there are sets $\mathcal{S}_\ell \subset \{w^1,w^2,\dotsc,w^t\}$, matrices $M^\ell \in \real^{N_{\ell} \times N_0}$ and vectors $z^\ell\in \real^{N_\ell}$ for $\ell=1,\dotsc,L$ such that
	\begin{itemize}
		\item[(i)] $|\mathcal{S}_\ell| \geq 3m\cdot  (N_{\ell}+1)  \dotsb  (N_{L-1}+1)$,
		\item[(ii)] $\mathcal{S}_\ell = \lbrace w^{s_\ell}, w^{s_\ell+1},\dotsc, w^{s_\ell+n_\ell}\rbrace$ for some $s_\ell, n_\ell \in \mathbb{N}$.
		\item[(iii)] $\varphi(w) = W^L \rho W^{L-1} \rho W^{L-2} \dotsc W^{\ell+1} \rho (M^\ell w+z^\ell)$ whenever $w \in \mathcal{S}_\ell$, where the $W^i$ are affine maps and $\rho$ is applied coordinatewise.
	\end{itemize}  
	The induction base is obvious by taking $\mathcal{S}_1=\mathcal{W}$, $M^1 = W^1$, and $z^1 = b^1$. The induction step will follow with the help of Lemma \ref{lem:NNReduce}. Indeed, assuming the existence of $\mathcal{S}_{\ell}$, $M^\ell$, and $z^\ell$ for some $\ell < L$, we apply Lemma \ref{lem:NNReduce} with $B = A^{\ell+1}, A = M^\ell, R = \mathcal{S}_\ell$ and $w = z^\ell$ to obtain some set $\mathcal{S}_{\ell+1}$, a matrix $M^{\ell+1}$ and a vector $v^{\ell+1}$ for which
	$A^{\ell+1} \rho(M^\ell w + z^\ell) = M^{\ell+1}w + v^{\ell+1}$ for $w \in \mathcal{S}_{\ell+1}$, and thus $W^{\ell+1} \rho(M^\ell w  + z^\ell) = M^{\ell+1}w + z^{\ell+1}$, where we set $z^{\ell+1} = v^{\ell+1} + b^{\ell+1}$. With the completed induction in hand, the proof of the claim follows by setting $\mathcal{S} = \mathcal{S}_L$, $s = s_L$, $n = n_L$, $M = M^L$ and $z = z^L$.
	
	Using the claim, we can now complete the proof of Lemma \ref{lem:NoNN}. Indeed, define the disjoint sets $\mathcal{S}^>$, $\mathcal{S}^<$ as follows:
	\begin{equation*}
		\mathcal{S}^> = \{w \in \mathcal{S} \, \vert \, g(\varphi(w)) \geq 1/2\}, \quad \mathcal{S}^<= \{w \in \mathcal{S} \, \vert \, g(\varphi(w)) < 1/2\}
	\end{equation*}
	For any $w \in \mathcal{S}$, we have $\varphi(w) = Mw + z$. Furthermore, any such $w$ has $w_{2} = w_{3} = \dotsb = w_{N} = 0$. Therefore $\varphi(w) = M_{1,1}w_1 + z$. In particular, $g \circ \varphi$ restricted to $\mathcal{S}$ is monotonic in the first coordinate of vectors in $\mathcal{S}$. This implies that 
	\begin{equation*}
		\mathcal{S}^> = \{w^{k_1},w^{k_1 + 1}, w^{k_1 + 2}, \dotsc ,w^{k_1 + t_1-1}\},\quad  \mathcal{S}^< = \{ w^{k_2}, w^{k_2 + 1}, w^{k_2 + 2} , \dotsc w^{k_2 + t_2-1}\}
	\end{equation*}
	for some $k_1$ and $k_2$ and $t_1,t_2$ with $t_1 + t_2 = |\mathcal{S}|\geq  3m$. Furthermore, by \ref{item:differingf} and the fact that the range of $f$ is the set $\{0,1\}$, we must have $f(w^i) = 1$ for all even $i$ and $f(w^i) = 0$ for all odd $i$ or $f(w^i) = 0$ for all even $i$ and $f(w^i) = 1$ for all odd $i$. We will consider these two cases separately
	
	\textbf{Case 1: $f(w^i) = 1$ for all even $i$ and $f(w^i) = 0$ for all odd $i$.}

	We define the sets
	\begin{equation*}
		\mathcal{S}^{E, < } = \{w^{i} \, \vert \, w^{i} \in \mathcal{S}^<,\, i \text{ even}\}, \quad \mathcal{S}^{O, >} = \{w^{i} \, \vert \, w^{i} \in \mathcal{S}^>,\, i \text{ odd}\}
	\end{equation*}
	For $w \in \mathcal{S}^{E,<}$, we have $f(w) = 1$ and $g(\varphi(w)) < 1/2$, whence we obtain $|g(\varphi(w)) - f(w)|\geq 1/2$. Similarly, for $w \in \mathcal{S}^{O,>}$ we have 
	$f(w) = 0$ and $g(\varphi(w)) \geq 1/2$ and we thus obtain $|g(\varphi(w)) - f(w)|\geq 1/2$. We set $\mathcal{U} = \mathcal{S}^{E,>} \cup \mathcal{S}^{O,<}$ and conclude that for any $w \in \mathcal{U}$ we have $|f(w) - g(\varphi(w))| \geq 1/2$.
	
	The claim about the cardinality of $\mathcal{U}$ follows by noting that $|\mathcal{S}^{E,<}| \geq \ceil{(t_1-1)/2}$ and that $|\mathcal{S}^{O,>}| \geq \ceil{(t_2-1)/2}$.  Therefore (using the disjointedness of $\mathcal{S}^{E,>}$ and $\mathcal{S}^{O,<}$)
	\begin{align}
		|\mathcal{U}| = |\mathcal{S}^{E,>}| + |\mathcal{S}^{O,<}|&\geq \ceil{(t_1-1)/2} + \ceil{(t_2-1)/2} \notag\\&\geq \ceil{(t_1 -1 +  t_2 -1)/2} = \ceil{(t_1 + t_2)/2} - 1 \geq \ceil{3m/2} - 1 \geq m \label{eq:UCardinality}
	\end{align}
	\textbf{Case 2: $f(w^i) = 0$ for all even $i$ and $f(w^i) = 1$ for all odd $i$.}
	
	The proof here is similar to that of Case 1. This time however, we define the sets \begin{equation*}
		\mathcal{S}^{E, > } = \{w^{i} \, \vert \, w^{i} \in \mathcal{S}^>,\, i \text{ even}\}, \quad \mathcal{S}^{O, <} = \{w^{i} \, \vert \, w^{i} \in \mathcal{S}^>,\, i \text{ odd}\}
	\end{equation*}
	An analogous argument to the above allows us to conclude that 
	$|g(\varphi(w)) - f(w)| \geq 1/2$ for all $w \in \mathcal{U}$, where this time $\mathcal{U} = 	\mathcal{S}^{E, > } \cup \mathcal{S}^{O,<}$. The argument that $|\mathcal{U}| \geq m$ is identical to \eqref{eq:UCardinality} except we replace references to $\mathcal{S}^{E,<}$ with $\mathcal{S}^{E,>}$ and references to $\mathcal{S}^{O,>}$ with $\mathcal{S}^{O,<}$.
\end{proof}
\subsection{Proof of Theorem \ref{theorem:NNCV}}
We require two further lemmas specific to the proof of Theorem \ref{theorem:NNCV}. These are stated as Lemma \ref{lem:NN-distr-lemma} and \ref{lemma:NNExistence}.
\begin{lemma}\label{lem:NN-distr-lemma}
	For $\gamma\in(1,2)$, define the probability distribution $\mathcal{P}=\{\mathrm{p}_j\}_{j=1}^\infty$ on $\mathbb{N}$ by $\mathrm{p}_{2j-1}=\mathrm{p}_{2j}=\frac{1}{2}C_\zeta(\gamma) j^{-\gamma}$, for $j\in\mathbb{N}$, where $C_\zeta(\gamma):=\big(\sum_{j=1}^\infty j^{-\gamma}\big)^{-1}$ is a normalising factor.
	
	Fix $\theta \in\mathbb{N}$ and let $X_1$, $X_2$,\dots, $X_\theta$ be i.i.d. random variables in $\mathbb{N}$ distributed according to $\mathcal{P}$. Next, consider the random set whose elements are the values of $X_1$ , $X_2$, \dots, $X_\theta$ and enumerate it as $S=\{Z_1,Z_2,\dotsc, Z_N\}$ with $Z_1<Z_2<\dotsb < Z_N$ (note that $N$, the number of distinct elements of $S$, is an integer-valued random variable such that $N\leq \theta$). Then, setting $c_1 =(1-e^{-C_\zeta(\gamma)})/2$ and $c_2=C_\zeta(\gamma)/(\gamma-1)$, we have
	\begin{itemize}
		\item[(i)]  $\pr\big(N\geq c_1 \theta^{1/\gamma}\big)\geq 1- c_1^{-2}\theta^{-(2/\gamma-1)}$, 
		\item[(ii)] $\pr(\max S \leq n)\geq 1- c_2\, \theta{\floor{n/2}}^{1-\gamma}$, for all $n\in\mathbb{N}$, and 
		\item[(iii)] $\pr\left(\left.\sum_{j=1}^{N-1} \chi_{\{Z_{j+1}-Z_{j}\text{ odd}\}}  \leq n/5 \,\right\rvert\, N=n \right) \leq e^{-n/100}$, for all integers $n$ such that $10\leq n \leq \theta$.
	\end{itemize}
\end{lemma}
\begin{proof}
	Throughout this proof we will use the convention that for a random variable $Y: \Omega \to \mathcal{E}$ the notation $\{Y = \mu\}$ for $\mu \in \Omega$ means the set $\{\tau \in \Omega \, \vert \, Y(\tau) = \mu\}$.
	
	For item (i), define the random variable $M_\theta$ to be the number of different unique values taken by the random variables $\ceil{X_1/2}$ ,\dots, $ \ceil{X_\theta/2}$ and note that
	$
	\pr\left( N< \beta \right)\leq \pr (M_r < \beta )  
	$, for $\beta \in\mathbb{R}$. Now, as the random variables $\ceil{X_j/2}$, $j=1,\dots,r$, are i.i.d. and distributed according to the zeta distribution with parameter $\gamma$,  it follows from \cite[Lem. 4, Lem. 3]{Zakhrevskaya2001} that 
	$
	\mathbb{E}[M_\theta]> (1-e^{-C_\zeta(\gamma)}) \theta^{1/\gamma}
	$ and
	$
	\sigma^2:=\mathrm{Var}[M_\theta]\leq \mathbb{E}[M_\theta] \leq \theta
	$,
	and hence Chebyshev's inequality yields
	\begin{align*}
		&\pr\left( N< \frac{1-e^{-C_\zeta(\gamma)}}{2} \theta^{1/\gamma} \right)\leq \pr \left(M_r < \frac{1-e^{-C_\zeta(\gamma)}}{2} \theta^{1/\gamma}  \right) \\
		\leq\;&  \pr \left(|M_\theta-\mathbb{E}[M_\theta]| >\frac{1-e^{-C_\zeta(\gamma)}}{2\sigma} \theta^{1/\gamma} \cdot \sigma \right)\leq  \left(\frac{1-e^{-C_\zeta(\gamma)}}{2\sigma} \theta^{1/\gamma} \right)^{-2}\leq \frac{4 \theta^{-(2/\gamma-1)}}{(1-e^{-C_\zeta(\gamma)})^{2}} ,
	\end{align*}
	which implies item (i).
	
	The proof of item (ii) is simple. Note that $\{\max S\leq n\}=\bigcap_{j=1}^r\{X_j\leq n\}$ and, for each $j$,
	\begin{align*}
		\pr(X_j \leq n)=\sum_{j=1}^n \mathrm{p}_j &\geq \sum_{j=1}^{\floor{n/2}} C_\zeta(\gamma)j^{-\gamma}\geq 1- C_\zeta(\gamma)\int_{\floor{n/2}}^\infty t^{-\gamma}\mathrm{d}t \geq 1 - \frac{C_\zeta(\gamma)}{\gamma-1} \floor{n/2}^{1-\gamma},
	\end{align*}
	and hence, as the $X_j$ are independent, 
	\begin{equation*}
		\pr(\max S\leq n)= \pr(X_j \leq n)^\theta\geq \left(1 - \frac{C_\zeta(\gamma)}{\gamma-1} \floor{n/2}^{1-\gamma}\right)^\theta\geq 1- \frac{C_\zeta(\gamma)}{\gamma-1} \theta\floor{n/2}^{1-\gamma}
	\end{equation*}
	where the last inequality follows by Bernoulli's inequality.

	Item (iii) is somewhat more involved. We start by outlining the strategy: the set $S$ may contain pairs of the form $(Z_{j},Z_{j+1})=(2i-1,2i)$, i.e., an odd natural number followed by the next even one. We will condition on the set of $j$ where $(Z_j,Z_{j+1})$ is such a pair, as well as the specific value of $Z_j$. 
	
	More precisely, for fixed sets $\mathcal{I}$ and $\mathcal{J}$ with $|\mathcal{I}| = |\mathcal{J}|$, enumerated by 
	\[\mathcal{I} = \{i_1,i_2,\dotsc,i_m\} \text{ and } \mathcal{J} = \{j_1,j_2,\dotsc,j_m\},\] let $\mathcal{A} = \{1,\dots, N\} \setminus\big( \mathcal{J} \cup(\mathcal{J}+1)\big)$ where $\mathcal{J}+1:=\{j+1\,\vert\,j\in\mathcal{J}\}$.
	We will condition on the event $F_{\mathcal{I},\mathcal{J}}$ which occurs precisely when $N = n$, $(Z_{j_\ell},Z_{j_\ell+1})=(2i_\ell-1,2i_\ell)$ for $\ell \in \{1,2,\dotsc,m\}$, and, on the indices in $\mathcal{A}$, the set $S$ contains no odd-even pairs, i.e.,  $(Z_{a},Z_{a+1})\notin\{ (2i-1,2i)\,\vert\, i\in \mathbb{N}\}$ for all $a \in \mathcal{A}$ with $a < n$ and $(Z_{a-1},Z_{a})\notin\{ (2i-1,2i)\,\vert\, i\in \mathbb{N}\}$ for all $a \in \mathcal{A}$ with $a > 1$. With varying $\mathcal{I}$ and $\mathcal{J}$, these sets $F_{\mathcal{I},\mathcal{J}}$ partition the event $\{N=n\}$.
	
	The intuition behind this construction is as follows: conditional on $F_{\mathcal{I},\mathcal{J}}$, whenever $j \in \mathcal{J}$ we have $Z_{j+1} - Z_j = 1$ and hence $\chi_{\{Z_{j+1}-Z_{j}\text{ odd}\}} = 1$. Thus for sets $\mathcal{J}$ with $|\mathcal{J}| \geq n/5$, we are done. If instead $|\mathcal{J}|$ is small then $|\mathcal{A}|$ will be relatively large. For $a \in \mathcal{A}$, we will argue that every $Z_a$ has equal probability of being an odd number or the even number following it, owing to the assumption that $\mathrm{p}_{2i-1}=\mathrm{p}_{2i}$ and the assumption that if $a < n$ then $(Z_{a},Z_{a+1})\notin\{ (2i-1,2i)\,\vert\, i\in \mathbb{N}\}$ and if $a > 1$ then $(Z_{a-1},Z_{a})\notin\{ (2i-1,2i)\,\vert\, i\in \mathbb{N}\}$. 
	
	This will allow us to conclude that the indicator random variables $\chi_{\{Z_{a}\text{ odd}\}}$ for $a\in\mathcal{A}$ are independent symmetric Bernoulli random variables (that is to say, they take the values $1$ and $0$ each with probability $1/2$). The desired bound will follow by an application of Hoeffding's inequality.
	
	We are now ready to present the formal proof. If $\theta<10$ there is nothing to prove, so assume that $\theta\geq 10$ and fix an $n$ such that $10\leq n\leq \theta$. Consider arbitrary sets $\mathcal{I}\subset\mathbb{N}$ and $\mathcal{J}\subset\{1,\dots,n-1\}$ so that 
	\begin{equation}\label{eq:prob-lemma-partition}
		m:=|\mathcal{I}|=|\mathcal{J}|<n \text{ and } \mathcal{J}\cap(\mathcal{J}+1)=\varnothing,
	\end{equation}
	and define $\mathcal{A}:=\{1,\dots, N\} \setminus\big( \mathcal{J} \cup(\mathcal{J}+1) \big)$. Enumerate $\mathcal{I}=\{i_{1},\dotsc,i_{m}\}$ with $i_{1}<\dotsb<i_{m}$, $\mathcal{J}=\{j_1,\dotsc,j_m\}$ with $j_1<\dotsb<j_m$, and $\mathcal{A}=\{a_1,\dotsc,a_{n-2m}\}$ with $a_1<\dotsb<a_{n-2m}$ and define the event
	\begin{equation*}
		\begin{split}
			F_{\mathcal{I},\mathcal{J}}=\{N=n\}\cap \bigcap_{\ell=1}^m\{(Z_{j_\ell}, Z_{j_\ell+1})&=(2i_\ell-1,2i_\ell)\} \cap \bigcap_{\substack{a\in\mathcal{A}, a<n\\ i\in\mathbb{N}}}\{(Z_{a},Z_{a+1})\neq(2i-1,2i)\}\\
			&\qquad \cap \bigcap_{\substack{a\in\mathcal{A}, 1<a \leq n\\ i\in\mathbb{N}}}\{(Z_{a-1},Z_{a})\neq(2i-1,2i)\}.
		\end{split}	
	\end{equation*}
	Note that, for every $n\in\mathbb{N}$, we have
	\begin{equation}\label{eq:n-union-omegas-probs}
		\{N=n\}={\bigcup_{ \substack{\mathcal{I}\subset\mathbb{N}, \mathcal{J}\subset\{1,\dots,n-1\}\\\text{satisfying \eqref{eq:prob-lemma-partition}}}}} F_{\mathcal{I},\mathcal{J}},
	\end{equation}
	i.e., the events $F_{\mathcal{I},\mathcal{J}}$ for different $\mathcal{I}$ and $\mathcal{J}$ partition the event $\{N=n\}$, and thus our strategy will be to prove the bound
	$
	\pr\left(\left.\sum_{j=1}^{N-1} \chi_{\{Z_{j+1}-Z_{j}\text{ odd}\}}  \leq n/5 \,\right\rvert\, F_{\mathcal{I},\mathcal{J}}  \right) \leq e^{-n/100}
	$
	for each of these events.
	
	The argument relies on bounding from below the number of indices $j$ such that $Z_{j+1}-Z_{j}$ is odd. For 
	$j\in \mathcal{J}$ this will be easy, as $Z_{j+1}-Z_j=2i_j-(2i_j-1)=1$ is always odd, by definition of $F_{\mathcal{I},\mathcal{J}}$. For $j\in \mathcal{A}$, we will need the following claim which we prove last.
	
	{\bf Claim:} For any $\mathcal{I}$, $\mathcal{J}$, and $\mathcal{A}$ as above, the indicator random variables $\chi_{\{Z_{a}\text{ odd}\}}$, $a\in\mathcal{A}$, conditional on $F_{\mathcal{I},\mathcal{J}}$ are independent symmetric Bernoulli variables.
	
	Armed with the claim, the counting argument is as follows. Note that, on the event  $F_{\mathcal{I},\mathcal{J}}$, for $k\in\{1,\dots, n-2m-1\}$ such that $a_{k+1}>a_k+1$, we have that $\{Z_{a_{k}},\dots, Z_{a_{k+1}}\}=\{Z_{a_k}, 2i_{t}-1,2i_{t},2i_{t+1}-1,2i_{t+1},\dotsc, 2i_{t+s-1} -1,2i_{t+s-1},Z_{a_{k+1}}\}$ for some $t \in \{1,2,\dotsc,m\}$ and where $s=|\mathcal{J}\cap\{a_k,\dots, a_{k+1}-1\}|$. Hence
	\begin{align}
		\sum_{\ell=a_k}^{a_{k+1}-1} \chi_{\{Z_{\ell+1}-Z_{\ell}\text{ odd}\}}&\geq \chi_{\{2i_{t}-1-Z_{a_k}\text{ odd}\}} + \sum_{\ell=0}^{s-1} \chi_{\{ (2i_{t+\ell}) - (2i_{t+\ell}-1) \text{ odd}\}} + \chi_{\{Z_{a_{k+1}}-2i_{t+s-1}\text{ odd}\}}\notag\\
		& = \chi_{\{Z_{a_k}\text{ even}\}} +  |\mathcal{J}\cap\{a_k,\dots, a_{k+1}-1\}| + \chi_{\{Z_{a_{k+1}}\text{ odd}\}}\notag \\
		& \geq |\mathcal{J}\cap\{a_k,\dots, a_{k+1}-1\}| + \chi_{\{Z_{a_{k+1}}-Z_{a_k}\text{ even}\}},\label{eq:prob-even-run}
	\end{align}
	where we used the simple observation that $\chi_{\{Z_{a_k}\text{ even}\}} + \chi_{\{Z_{a_{k+1}}\text{ odd}\}}\geq \chi_{\{Z_{a_{k+1}}-Z_{a_k}\text{ even}\}} $. This motivates defining random variables $E_{a_k}$ with $k\in\{1,\dots, n-2m-1\}$ conditioned on the event $F_{\mathcal{I},\mathcal{J}}$ according to
	\begin{align*}
		E_{a_k}& =\begin{cases}
			1,& Z_{a_{k}+1}-Z_{a_k}\text{ is odd}\\
			0, & Z_{a_{k}+1}-Z_{a_k}\text{ is even}
		\end{cases}, \quad \text{for $k$ s.t. $a_{k+1}=a_k+1$, and}\\
		E_{a_k}&=\begin{cases}
			0,& Z_{a_{k+1}}-Z_{a_k}\text{ is odd}\\
			1, & Z_{a_{k+1}}-Z_{a_k}\text{ is even}
		\end{cases}, \quad \text{for $k$ s.t. $a_{k+1}>a_k+1$,}
	\end{align*}
	which, as a consequence of the Claim, are themselves independent symmetric Bernoulli random variables.
	Thus, writing $U:=\sum_{k=1}^{N-1} \chi_{\{Z_{k+1}-Z_{k}\text{ odd}\}}$, on the event $F_{\mathcal{I},\mathcal{J}}$ we have
	\begin{align}
		U&=\sum_{\ell < a_1\text{ or }\, \ell\geq a_{n-2m}} \chi_{\{Z_{\ell+1}-Z_{\ell}\text{ odd}\}} +   \sum_{k=1 }^{n-2m-1}\sum_{\ell=a_k}^{a_{k+1}-1} \chi_{\{Z_{\ell+1}-Z_{\ell}\text{ odd}\}} \notag\\
		&\geq |\mathcal{J}\cap \{1,\dots, a_1-1\}| + |\mathcal{J}\cap \{a_{n-2m}, \dots, n\} | \notag\\
		&\qquad \qquad +  \sum_{\substack{1\leq k\leq n-2m-1\\a_{k+1}=a_k+1}} \chi_{\{Z_{a_{k}+1}-Z_{a_k}\text{ odd}\}} + \sum_{\substack{1\leq k\leq n-2m-1\\a_{k+1}>a_k+1}}\sum_{\ell=a_k}^{a_{k+1}-1} \chi_{\{Z_{\ell+1}-Z_{\ell}\text{ odd}\}} \notag\\
		&\geq  |\mathcal{J}\cap \{1,\dots, a_1-1\}| + |\mathcal{J}\cap \{a_{n-2m}, \dots, n\} | \notag\\
		&\qquad\qquad  + \sum_{\substack{1\leq k\leq n-2m-1\\a_{k+1}=a_k+1}}  E_{a_k} +  \sum_{\substack{1\leq k\leq n-2m-1\\a_{k+1}>a_k+1}}\left( |\mathcal{J}\cap \{a_k ,\dots, a_{k+1}-1\}| + E_{a_k} \right) \notag \\
		&=|\mathcal{J}| + \sum_{k=1}^{n-2m-1} E_{a_k}= m+ \sum_{k=1}^{n-2m-1} E_{a_k}, \label{eq:big-prob-calc}
	\end{align}
	where the second inequality is due to \eqref{eq:prob-even-run} and the penultimate equality follows from the observation that $|\mathcal{J}\cap \{a_k ,\dots, a_{k+1}-1\}| = 0$ whenever $a_{k+1} = a_k + 1$.
	
	Now, for sets $\mathcal{I}\subset\mathbb{N}$ and $\mathcal{J}\subset\{1,\dots,n-1\}$ satisfying \eqref{eq:prob-lemma-partition}
	as well as $m=|\mathcal{I}|=|\mathcal{J}|\leq n/5$, we have that \eqref{eq:big-prob-calc} implies $U\geq \sum_{k=1}^{n-2m-1} E_{a_k} $, which together with Hoeffding's inequality yields
	\begin{align*}
		\pr\Big(U \leq n/5 \,\Big\rvert\, F_{\mathcal{I},\mathcal{J}} \Big)& \leq \pr\Big( \sum_{k=1}^{n-2m-1} E_{a_k}  \leq n/5 \,\Big\rvert\, F_{\mathcal{I},\mathcal{J}} \Big)\notag\\
		&\leq  \exp\left(-2\Big(\frac{1}{2} - \frac{n/5}{n-2m-1}\Big)^{2}(n-2m-1) \right)\leq \exp \left(- n/100 \right)
	\end{align*}
	where in the last inequality we used $n-2m-1\geq n/2$ (recall that $n\geq 10$). On the other hand, in the case when $m=|\mathcal{I}|=|\mathcal{J}|> n/5$ we have $\pr\Big(U \leq n/5 \,\Big\rvert\, F_{\mathcal{I},\mathcal{J}} \Big)=0$ directly from \eqref{eq:big-prob-calc}.

	Therefore, we have shown that for any $\mathcal{I}$, $\mathcal{J}$ satisfying \eqref{eq:prob-lemma-partition}, $\pr\Big(U \leq n/5 \,\Big\rvert\, F_{\mathcal{I},\mathcal{J}} \Big) \leq \exp (-n/100)$ and so using \eqref{eq:n-union-omegas-probs}
	\begin{align*}
		\pr\Big( U  \leq n/5, N=n \Big)&= \sum_{ \substack{\mathcal{I}\subset\mathbb{N}, \mathcal{J}\subset\{1,\dots,n-1\}\\\text{satisfying \eqref{eq:prob-lemma-partition}}}}  \pr\Big( U\leq n/5 \,\Big\rvert\, F_{\mathcal{I},\mathcal{J}}  \Big) \pr(F_{\mathcal{I},\mathcal{J}})\\
		&\leq \exp \left(- n/100 \right) \pr\Bigg(\bigcup_{ \substack{\mathcal{I}\subset\mathbb{N}, \mathcal{J}\subset\{1,\dots,n-1\}\\\text{satisfying \eqref{eq:prob-lemma-partition}}}} F_{\mathcal{I},\mathcal{J}} \Bigg)  = \exp \left(- n/100 \right) \pr(N=n),
	\end{align*}
	which yields the desired bound after dividing both sides by $\pr(N=n)$.

	It remains to prove the Claim. To this end, fix $n$,  $\mathcal{I}=\{i_1<\dots<i_m\}$, $\mathcal{J}=\{j_1<\dots<j_m\}$, and $\mathcal{A}=\{a_1<\dots<a_{n-2m}\}$ satisfying \eqref{eq:prob-lemma-partition}. Then, conditional on $F_{\mathcal{I},\mathcal{J}}$ we can write $Z_a=2\ceil{Z_a/2} - \chi_{\{Z_{a}\text{ odd}\}}$, for $a\in\mathcal{A}$, where the $\chi_{\{Z_{a}\text{ odd}\}}$ are random variables taking values in $\{0,1\}$ and the
	$\ceil{Z_a/2}$ are random variables taking values in $\mathbb{N}\setminus\mathcal{I}$ and moreover $\ceil{Z_{a_1}/2}<\dots< \ceil{Z_{a_{n-2m}}/2}$. Now, for a set $\mathcal{U}=\{u_1<\dots<u_{n-2m}\}\subset \mathbb{N}\setminus \mathcal{I}$ denote $\mathring F_{\mathcal{U}}=\bigcap_{j=1}^{n-2m}\{\ceil{Z_{a_j}/2}=u_j\}$ so that for any $b\in\{0,1\}^{n-2m}$
	\begin{align}
		&\pr\left(\{\chi_{\{Z_{a_1}\text{ odd}\}}=b_1,\; \dots\; , \chi_{\{Z_{a_{n-2m}}\text{ odd}\}}=b_{n-2m}\}\,\Big\rvert\, F_{\mathcal{I},\mathcal{J}} \right)\notag\\
		=\,&\sum_{\mathcal{U}\subset \mathbb{N}\setminus \mathcal{I}}\pr\left(\{\chi_{\{Z_{a_1}\text{ odd}\}}=b_1,\; \dots\; , \chi_{\{Z_{a_{n-2m}}\text{ odd}\}}=b_{n-2m}\}\,\Big\rvert\, F_{\mathcal{I},\mathcal{J}}\cap \mathring F_{\mathcal{U}} \right) \pr( \mathring F_{\mathcal{U}} \, \vert \, F_{\mathcal{I},\mathcal{J}}) \notag\\
		=\,&\sum_{\mathcal{U}\subset \mathbb{N}\setminus \mathcal{I}}\pr\left(\{\chi_{\{Z_{a_1}\text{ odd}\}}=0,\; \dots\; , \chi_{\{Z_{a_{n-2m}}\text{ odd}\}}=0\}\,\Big\rvert\, F_{\mathcal{I},\mathcal{J}}\cap \mathring F_{\mathcal{U}} \right) \pr( \mathring F_{\mathcal{U}} \, \vert \, F_{\mathcal{I},\mathcal{J}}) \label{eq:use-p2j=p2j-1}\\
		=\,&\pr\left(\{\chi_{\{Z_{a_1}\text{ odd}\}}=0,\; \dots\; , \chi_{\{Z_{a_{n-2m}}\text{ odd}\}}=0\}\,\Big\rvert\, F_{\mathcal{I},\mathcal{J}} \right),\notag
	\end{align}
	where in \eqref{eq:use-p2j=p2j-1} we used the fact that $\mathrm{p}_{2j-1}=\mathrm{p}_{2j}$, for all $j\in\mathbb{N}$. It hence follows that the $\chi_{\{Z_{a_j}\text{ odd}\}}$, $1\leq j\leq n-2m$, conditional on $F_{\mathcal{I},\mathcal{J}}$ are independent symmetric Bernoulli variables, establishing the Claim and thus completing the proof.
\end{proof}

\begin{lemma}\label{lemma:NNExistence}
	Fix an even $K\in \mathbb{N}$ and let $\{\alpha_j\}_{k=1}^{K}$ be such that $0< \alpha_{k+1}<\alpha_k<1$ for all $1 \leq k \leq K-1$. 
	Furthermore, let $N_0\in\mathbb{N}$. Then there exists a neural network $\psi:\real^{N_0}\to \real$ with the ReLU nonlinearity $\rho(t)=\max\{0,t\}$ such that
	\begin{equation}\label{eq:NNPropertyConstruction}
		\psi(x) =\begin{cases}
			0 &\text{ whenever } x_1 \in [\alpha_{k},\alpha_{k - 1}] \text{ with } k \equiv 2  \mod 4\\
			1 & \text{ whenever } x_1 \in [\alpha_{k},\alpha_{k-1}] \text{ with } k \equiv  0 \mod 4
		\end{cases},\qquad \text{for all }x\in\real^{N_0} \text{ and } k\in \{2,3,\dotsc,K\}.
	\end{equation}	
\end{lemma}
\begin{proof}
	We may w.l.o.g. assume that $K$ is divisible by 4. Indeed, if $K$ is not divisible by $4$, we can extend the sequence $\{\alpha_k\}_{k=1}^{K}$ by adjoining two new elements (say $\alpha_{K}/2$ and $\alpha_K/4$) at the end of the sequence. We additionally set $\alpha_{K+1}=0$ for convenience. Now, for $\ell\in\{1,\dots, K/4\}$, define the single-layer neural network
	\begin{equation*}
		\begin{aligned}
			\psi_\ell(x)&=\left(\alpha_{4\ell-2}-\alpha_{4\ell-1}\right)^{-1} \big( \rho(\alpha_{4\ell-2} -x_1)-  \rho(\alpha_{4\ell-1} -x_1) \big)\\
			&\qquad\qquad  - \left(\alpha_{4\ell}-\alpha_{4\ell+1}\right)^{-1} \big( \rho(\alpha_{4\ell} -x_1)-  \rho(\alpha_{4\ell+1} -x_1) \big),\quad \text{for }x\in\real^{N_0}.
		\end{aligned}
	\end{equation*}
	One now easily verifies that $\psi_{\ell}(x)=1$ whenever $x_1\in[\alpha_{4\ell}, \alpha_{4\ell-1}]$ and $\psi_{\ell}(x)=0$ whenever $x_1\in \real\setminus (\alpha_{4\ell+1}, \alpha_{4\ell-2})$. Hence, setting $\psi(x)=\sum_{k=1}^{K/4}\psi_{\ell}(x)$ yields the desired network.
\end{proof}

We are now in a position to prove theorem \ref{theorem:NNCV}:
\begin{proof}[Proof of Theorem \ref{theorem:NNCV}]
	
	We begin by defining the sets $\mathcal{C}_1$ and $\mathcal{C}_2$. Let $\mathcal{C}_1 = \{ f_a: \real^d \to [0,1] \, \vert \, a \in [1/2,1]\}$, where $f_a$ is defined as in \eqref{eq:TheFs}. Since all norms on finite dimensional vector spaces are equivalent, let $D > 0$ be such that $\|\cdot \| \leq D \|\cdot\|_1$. To define the set of distributions we first set $\delta =\epsilon/(2D)$. For each $\kappa \in [1/4,3/4]$, define the distribution $\mathcal{D}_\kappa$ on $[0,1]^{N_0}$
	\begin{equation*}
		X\sim \mathcal{D_{\kappa} }\quad \iff \quad \pr(X=x)=\begin{cases}
			\mathrm{p}_k & \text{if }x= x^{k,\delta}\\
			0 &\text{otherwise}
		\end{cases}.
	\end{equation*}
	where $\mathrm{p}_{2j-1}=\mathrm{p}_{2j}=C_\zeta(3/2)j^{-3/2}$ for $j\in\mathbb{N}$ and $x^{k,\delta}$ is defined according to \eqref{eq:theNNexes}. We set $\mathcal{C}_2 = \{\mathcal{D}_\kappa \, \vert \, \kappa \in [1/4,3/4]\}$.
	
	Let $c_1$, $c_2$, and $C_\zeta(3/2)$ be the constants defined in Lemma \ref{lem:NN-distr-lemma} with $\gamma$ set to $3/2$. We choose the constant $C$ so that each of the following hold:
	\begin{align}
		C& \geq 4^3c_1^{-6}, \label{eq:ProbNNFirstCBound}\\
		C &\geq 200 \log(8)^{3/2} c_1^{-3/2}, \text{ and}\label{eq:ProbNNSecondCBound}\\
		C &\geq 4\cdot (8c_2)^{2}.\label{eq:ProbNNThirdCBound}
	\end{align}

	Fix $a \in [1/2,1]$ so that $f_a \in \mathcal{C}_1$ and $\kappa \in [1/4,3/4]$ so that $\mathcal{D}_{\kappa} \in \mathcal{C}_2$. Let  $\mathcal{T} = \{x^1, \hdots, x^r\}$ and $\mathcal{V} = \{y^1, \hdots, y^s\}$ be the random multisets drawn from this distribution as in the statement of the theorem. Then by the definition of the distribution $\mathcal{D}_\kappa$, we can write (after removing repetitions and reordering) $\mathcal{T} \cup \mathcal{V}$ as  $	S:= \mathcal{T} \cup \mathcal{V} = \{ x^{Z_1,\delta},x^{Z_2,\delta},x^{Z_3,\delta}, \dotsc ,x^{Z_N,\delta}\}$ where the random variable $N$ satisfying $N \leq r+s$ is the number of unique elements in $\mathcal{T}\cup \mathcal{V}$ and where $Z_1 <Z_2 < \dotsb < Z_{N}$. For shorthand, we also set $z^j = x^{Z_j,0}$ for $j=1,2,\dotsc,N$.
	
	Since $C/2 \geq 2 \cdot (8c_2)^2$ (by \eqref{eq:ProbNNThirdCBound}) and $C(r\vee s)^2/(2p^2) \geq 4^3c_1^{-6}/2 \geq 2$ (by \eqref{eq:ProbNNFirstCBound} and the facts that $(r \vee s)/p \geq 1$ and $c_1^{-1} \geq 1$) we obtain 
	\[\frac{C(r\vee s)^2}{p^2} = \frac{C(r\vee s)^2}{2p^2} + \frac{C(r\vee s)^2}{2p^2} \geq 2\cdot \frac{(8c_2)^2(r\vee s)^2}{p^2}  +2 \geq 2\ceil*{\left(\frac{8c_2 (r \vee s)}{p}\right)^2+1} - 2\] and thus item (ii) of Lemma \ref{lem:NN-distr-lemma} with $\gamma = 3/2$ yields
	\begin{align}
		&\pr\left(\max\{k\in\mathbb{N}\,\vert\, x^{k,\delta}\in\mathcal{T} \cup \mathcal{V} \}  \leq \ceil*{\frac{C (r\vee s)^2}{\mathrm{p}^2}} \right)=\pr\left(Z_N  \leq \ceil*{\frac{C (r\vee s)^2}{\mathrm{p}^2}} \right) \notag \\
		&\geq \pr\left(Z_N  \leq 2\ceil*{\left(\frac{8c_2 (r \vee s)}{p}\right)^2+1} - 2  \right)\notag
		\geq 1- \frac{c_2 (r+s)}{ \floor*{\ceil*{\left(\frac{8c_2 (r \vee s)}{p}\right)^2+1} - 1}^{\frac{1}{2}}} \notag\\
		&\geq1- \frac{c_2 (r+s)}{ (8 c_2 (r \vee s)/\mathrm{p})} \geq 1-\mathrm{p}/4.\label{eq:NN-separation-prob1}
	\end{align}
Writing $N_{\text{prod}}:= (N_1+1)\cdots(N_{L-1}+1)$, by the assumptions \eqref{eq:r-NN-lower-bound} and \eqref{eq:ProbNNSecondCBound}, we obtain
	\[
	\floor{c_1 (r+s)^{2/3}} \geq \floor{C^{2/3}c_1 qN_{\text{prod}}} \geq \floor{200^{2/3} \log(8)qN_{\text{prod}}} \geq 30 qN_{\text{prod}}.
	\] 
	Therefore we can apply item (iii) of Lemma \ref{lem:NN-distr-lemma} to see that
	\begin{align}
		&\pr\Big(\sum_{i=1}^{N-1}  \chi_{\{ f_a(z^{i+1}) \neq f_a(z^{i})   \}} >  6 qN_{\text{prod}} \Big) = \pr\Big(\sum_{i=1}^{N-1}  \chi_{\{ Z_{i+1} - Z_i \text{ odd }   \}} >  6 qN_{\text{prod}} \Big)\notag\\&\geq  \!\!\!\!\!\!\!\!\!\!\sum_{n=\floor{c_1(r+s)^{\frac{2}{3}}\!}}^{r+s} \!\!\!\!\!\!\pr\Big( \sum_{i=1}^{n-1} \chi_{\{ Z_{i+1} - Z_i \text{ odd }   \}}>  \frac{n}{5}  \,\Big\rvert\, N = n\Big)\pr(N=n)\notag\\
		&\geq  \!\!\!\!\!\!\!\!\!\!\sum_{n=\floor{c_1(r+s)^{\frac{2}{3}}\!}}^{r+s} \!\!\!\!\!\!\exp\left(-\frac{n}{100}\right)\pr(N=n)\notag\\ 
		&\geq \left[1- \exp\left(-\floor*{\frac{c_1 (r+s)^{\frac{2}{3}}} {100}}\right) \right]\cdot \pr( N \geq \floor{c_1 (r+s)^{2/3}} )  \label{eq:NN-number-of-alterTemp}
	\end{align}
	where the application of Lemma \ref{lem:NN-distr-lemma} is justified by the bound $\floor{c_1 (r+s)^{2/3}} \geq 30 qN_{\text{prod}} \geq 10$ and the initial equality in the first line is justified by the fact that $f_a(z^i)$ depends only on the parity of $i$, a fact itself readily seen from the definition of $f_a$ and $z^i$.
	
	Now, by differentiating it is easy to see that the function $p \mapsto p\log(8/p)$ is increasing on $(0,1)$. Hence for $p < 1$ we have $p^{-2}\log(8) > p^{-1} \log(8)>\log(8/p)$ and so combining this with \eqref{eq:r-NN-lower-bound} and \eqref{eq:ProbNNSecondCBound} gives
	\begin{equation}\floor*{\frac{c_1 (r+s)^{2/3}}{100}}  \geq \floor*{\frac{c_1 C^{2/3}\mathrm{p}^{-2}}{100}} \geq  \floor*{\frac{200^{2/3} \mathrm{p}^{-2}\log(8)}{100}}\geq \mathrm{p}^{-2}\log(8) - 1 \geq \log(8/\mathrm{p}) - 1. \label{eq:NNCVTempBound}\end{equation} Furthermore, using item (i) of Lemma \ref{lem:NN-distr-lemma} with $\gamma=3/2$, we obtain $\pr\big(N \geq c_1 (r+s)^{2/3}\big)\geq 1- c_1^{-2}(r+s)^{-1/3}\geq1-\mathrm{p}/4,$ where the final bound follows because $r+s \geq Cp^{-3}$ (which, in turn, is due to the assumption \eqref{eq:r-NN-lower-bound}) and  \eqref{eq:ProbNNFirstCBound}. Using this result together with \eqref{eq:NNCVTempBound}  in \eqref{eq:NN-number-of-alterTemp} yields
	\begin{equation}
		\pr\Big(\sum_{i=1}^{N-1}  \chi_{\{ f_a(z^{j+1}) \neq f_a(z^{j})   \}} >  6 qN_{\text{prod}} \Big) > \left( 1- e\mathrm{p}/8\right)\left(1-\mathrm{p}/{4}\right)  > 1- \mathrm{p}/2\label{eq:NN-number-of-alter-2}
	\end{equation}
	
	Combining \eqref{eq:NN-separation-prob1} and \eqref{eq:NN-number-of-alter-2} we see that the probability that both 
	\begin{equation}\label{eq:NN-useful-event}
		\max\{k\in\mathbb{N}\,\vert\, x^{k,\delta}\in\mathcal{T}\cup \mathcal{V}\}   \leq \ceil*{ \frac{C(r\vee s)^2}{\mathrm{p}^2} } \quad \text{and }  \sum_{i=1}^{N-1} \chi_{\{ f_a(z^{i+1}) \neq f_a(z^{i})  \}} >  6 qN_{\text{prod}}
	\end{equation}
	occur is at least $1-(\mathrm{p}/4+ \mathrm{p}/2)> 1-  \mathrm{p}$. We will now proceed to show that each of (i) through (iii) listed as in the statement of Theorem \ref{theorem:NNCV} hold assuming that this event occurs.
	
	\textbf{Proof of (i): Success -- great generalisability }
	
	To see that $\mathcal{T}, \mathcal{V} \in \mathcal{S}^f_{\varepsilon((r\vee s)/\mathrm{p})}$, note that \eqref{eq:ProbNNSecondCBound} and $c_1^{-1} \geq 1$ yields $C^{2}t^{2}\geq (4\ceil{t}+3)(4\ceil{t}+4)$ for all $t\geq 1$. Applying this inequality with $t = C ((r\vee s)/\mathrm{p})^2\geq 1$, we deduce that \begin{align}\varepsilon\left[\frac{C(r\vee s)}{\mathrm{p}}\right] = C^{-2} \left( \frac{C(r\vee s)^2}{\mathrm{p}^2}\right)^{-2}&\leq \left[\left(4 \ceil*{ \frac{C(r\vee s)^2}{\mathrm{p}^2} } +3\right)\left(4 \ceil*{ \frac{C(r\vee s)^2}{\mathrm{p}^2} } +4\right)\right]^{-1}\notag \\&=\varepsilon'\left(\ceil*{ \frac{C(r\vee s)^2}{\mathrm{p}^2} }\right) \label{eq:EpsBoundByEpsPrime},\end{align} where $\varepsilon'(n)=[(4n+3)(4n+4)]^{-1}$. Therefore because we assume that $\max\{k\in\mathbb{N}\,\vert\, x^{k,\delta}\in\mathcal{T}\cup \mathcal{V} \}  \leq \ceil*{ \frac{C(r\vee s)^2}{\mathrm{p}^2} }$, Lemma \ref{lem:NN-separation-lemma} yields $\mathcal{T},\mathcal{V}\subset \{x^{1,\delta}, \dots, x^{ \ceil*{C(r\vee s)^2/\mathrm{p}^2} ,\delta}\}\in \mathcal{S}^{f_a}_{\varepsilon'(\ceil*{C(r\vee s)^2/\mathrm{p}^2})}\subset \mathcal{S}^{f_a}_{\varepsilon(C(r\vee s)/\mathrm{p})}$.

	The construction of $\phi $ satisfying \eqref{eq:NNCVCorrectClassification} is immediate: we take $\phi$ to be the neural network $\tilde{\varphi}$ defined in Lemma \ref{lem:NNUnstableExists}. We conclude that ${\phi}(x) = f_a(x)$ for all $x \in \mathcal{T} \cup \mathcal{V}$ (this establishes \eqref{eq:NNCVCorrectClassification}). Because ${\phi}(x) = f_a(x)$ for all $x \in \mathcal{T}$ and because $\mathcal{R} \in \mathcal{CF}_{r}$ we conclude that $\mathcal{R} \left(\{\phi({x}^j)\}_{j=1}^r,\{f({x}^j)\}_{j=1}^r\right) = 0$. Thus \eqref{eq:NNOptimisationProblem} holds, completing the proof of (i).

	\textbf{Proof of (ii): Any successful NN in $\mathcal{NN}_{\mathbf{N},L}$ -- regardless of architecture -- becomes universally unstable}
	
	Our next task will be to show that if $\hat\phi \in \mathcal{NN}_{\mathbf{N},L}$ and $g: \real \to \real$ is monotonic then there is a subset $\mathcal{\tilde T}\subset \mathcal{T} \cup \mathcal{V}$ of the combined training and validation set of size $|\mathcal{\tilde T}| \geq q$, such that there exist uncountably many universal adversarial perturbations $\eta \in \mathbb{R}^d$ so that for each $x \in \mathcal{\tilde T}$ equation \eqref{eq:NNCVIncorrectClassification} applies.

	To this end, note that \eqref{eq:NN-useful-event} implies that there exist natural numbers $k_1<k_2<\dots < k_{6qN_{\text{prod}}}$ such that $z^{k_i}_1>z^{k_{i+1}}_1$ and  $f_a(z^{k_i})\neq f_a(z^{k_{i+1}})$ for all  $i\in\{1,\dotsc, 6qN_{\text{prod}}-1\}$. Moreover, by the definition of $\mathcal{T}$, $\mathcal{V}$ and $S$ there exist $m_i$ such that $z^{k_i}_1 = x^{m_i,\delta}_1$ and such that $x^{m_i,\delta} \in \mathcal{T} \cup \mathcal{V}$. For such $i$ and any $\omega \in [0,\delta\wedge \varepsilon((r\vee s)/\mathrm{p}) )$, we define the vectors $w^{i,\omega} = z^{k_i}+\omega e_1$. We also define the sets 
	$\mathcal{W}^{\omega} := \{ w^{i,\omega} \, \vert \, i\in\{1,\dotsc, 6qN_{\text{prod}}\}\}$. 
	
	Because of the definition of $x^{k,0}$ given in \eqref{eq:theNNexes} and the definition of $z^{k_i}$, we have $z^{k_i}_2 = z^{k_i}_3 = \dotsb = z^{k_i}_d = 0$ and $z^{k_i} = x^{m_i,0}$. In particular, $\{z^{k_i} \, \vert \, i \in\{1,\dotsc, 6qN_{\text{prod}}\}\} = \{ x^{m_i,0} \, \vert \, i\in\{1,\dotsc, 6qN_{\text{prod}}\} \} \in  \mathcal{S}^{f_a}_{\varepsilon((r\vee s)/\mathrm{p})}$ where we have used Lemma \ref{lem:NN-separation-lemma} and the bound \eqref{eq:EpsBoundByEpsPrime}. Since $\|z ^{k_i} - w^{i,\omega}\|_{\infty} = \omega < \varepsilon((r\vee s)/\mathrm{p})$, we conclude that $f_a(z^{k_i}) = f_a(w^{i,\omega})$ for $i \in\{1,\dotsc, 6qN_{\text{prod}}\}$. Thus $f_a(w^{i,\omega}) = f_a(z^{k_i}) \neq f_a(z^{k_{i+1}}) = f_a(w^{i+1,\omega})$ for $i \in\{1,\dotsc, 6qN_{\text{prod}}-1\}$.
	
	We can now use Lemma \ref{lem:NoNN} to conclude that for each $\omega \in [0,\delta\wedge \varepsilon((r\vee s)/\mathrm{p}) )$ there exists a set $\mathcal{I}^{\omega}$ and a set $\mathcal{U}^{\omega}\subset \mathcal{W}^{\omega}$ with the following properties: 
	\begin{enumerate}
		\item $\mathcal{I}^{\omega} \subset \{1,2,\dotsc,6qN_{\text{prod}}\}$
		\item $\mathcal{U}^{\omega} = \{ w^{i,\omega} \, \vert \, i\in\mathcal{I}^{\omega}\}$
		\item For all $w \in \mathcal{U}^{\omega}$, $|g(\hat\phi(w)) - f_a(w)| \geq 1/2$.
		\item $|\mathcal{U}^{\omega}| \geq 2q$.
	\end{enumerate}
	By the pigeonhole principle and the finiteness of $\{1,2,\dotsc,6qN_{\text{prod}}\}$, there exists an uncountable set $\Omega \subset [0,\delta\wedge \varepsilon((r\vee s)/\mathrm{p}))$ such that for all $\omega \in \Omega$, $\mathcal{I}^{\omega}$ is independent of $\omega$. Let $\mathcal{I}$ denote this common value and let $\mathcal{I}_{E} := \{i \, \vert \, i \in \mathcal{I}, m_i \text{ even}\}$ and $\mathcal{I}_{O} := \{i \, \vert \, i \in \mathcal{I}, m_i \text{ odd}\}$. Note that $|\mathcal{I}| \geq 2q$; otherwise, $|\mathcal{U}^{\omega}| < 2q$ for some $\omega$. Therefore at least one of $|\mathcal{I}_E|\geq q$ or $|\mathcal{I}_O|\geq q$: we now split into two cases depending on which of these two sets has cardinality at least $q$.

	\textit{Case 1: $|\mathcal{I}_E|\geq q$.}
	
	In this case, we choose $\tilde{\mathcal{T}} = \{x^{m_i,\delta} \, \vert \, i \in \mathcal{I}_E\}$. For each $\omega \in \Omega$, define $\eta^{\omega} = (\omega,-\delta,0,\dotsc,0) \in \real^d$ and $\mathcal{H} = \{ \eta^{\omega} \, \vert \, \omega \in \Omega\}$. Then the set $\mathcal{H}$ is uncountable, for each $i \in \mathcal{I}_E$ and $\omega \in \Omega$ we have  $x^{m_i,\delta} + \eta^{\omega} = w^{i,\omega}$, $|g(\hat\phi(x^{m_i,\delta} + \eta^{\omega})) - f_a(x^{m_i,\delta} + \eta^{\omega})| = |g(\hat\phi(w^{i,\omega})) - f_a(w^{i,\omega})| \geq 1/2$ and $\|\eta\| \leq D\|\eta^{\omega}\|_1 = D(\omega + \delta) \leq 2D\delta \leq \epsilon$. Furthermore $|\supp(\eta^\omega)| = 2$. We conclude that \eqref{eq:NNCVIncorrectClassification} holds.

	\textit{Case 2: $|\mathcal{I}_O| \geq q$}
	
	In this case, we choose $\tilde{\mathcal{T}} = \{x^{m_i,\delta} \, \vert \, i \in \mathcal{I}_O\}$. For each $\omega \in \Omega$, define $\eta^{\omega} = (\omega,0,0,\dotsc,0) \in \real^d$ and $\mathcal{H} = \{ \eta^{\omega} \, \vert \, \omega \in \Omega\}$. Then the set $\mathcal{H}$ is uncountable, for each $i \in \mathcal{I}_O$ and $\omega \in \Omega$ we have  $x^{m_i,\delta} + \eta^{\omega} = w^{i,\omega}$, $|g(\hat\phi(x^{m_i,\delta} + \eta^{\omega})) - f_a(x^{m_i,\delta} + \eta^{\omega})| = |g(\hat\phi(w^{i,\omega})) - f_a(w^{i,\omega})| \geq 1/2$ and $\|\eta^{\omega}\| \leq D\|\eta^{\omega}\|_1 = D\omega \leq D\delta \leq \epsilon$. Furthermore $|\supp(\eta^\omega)| = 1$. We conclude that \eqref{eq:NNCVIncorrectClassification} holds.
	
	\textbf{Proof of (iii): Other stable and accurate NNs exist}
	
	Finally, we must show the existence of $\psi$, which we do with the help of Lemma \ref{lemma:NNExistence}. To this end, we set $K= \ceil{C ((r\vee s)/\mathrm{p})^2 } $ and  define $\{\alpha_j\}_{j=1}^{2K}$ by
	$
	\alpha_{2k-1}=x_1^{k,\delta}+ \varepsilon((r\vee s)/\mathrm{p})
	$,
	$
	\alpha_{2k}= x_1^{k,\delta}-\varepsilon((r\vee s)/\mathrm{p})
	$
	for $k=1,\dots, K $. We first claim that $0 < \alpha_{2K} < \alpha_{2K-1} < \dotsb < \alpha_2 < \alpha_1 < 1$.

	Because $C \geq 4^3$, $p \leq 1$ and $(r \vee s) \geq 1$ we have 
	\[
	\alpha_1 = \frac{a}{2-\kappa} + \frac{\mathrm{p}^4}{C^4 (r\vee s)^4} \leq \frac{1}{(2-3/4)} + \frac{1}{C^4} < 1
	\]
	and similarly we obtain $2\ceil{C ((r\vee s)/\mathrm{p})^2} + 1 - \kappa \leq 2(C ((r\vee s)/\mathrm{p})^2) + 2 - \kappa \leq 4(C ((r\vee s)/\mathrm{p})^2)$. Therefore
	\begin{align*}
		\alpha_{2K} = \frac{a}{2\ceil{C ((r\vee s)/\mathrm{p})^2} + 1 - \kappa} - \frac{\mathrm{p}^4}{C^4 (r\vee s)^4}   &\geq \frac{a}{4C ((r\vee s)/\mathrm{p})^2} - \frac{\mathrm{p}^4}{C^4 (r\vee s)^4} \\&\geq \frac{\mathrm{p}^2}{8C (r\vee s)^2} - \frac{\mathrm{p}^2}{4^{12}C (r\vee s)^2}>0
	\end{align*}
	
	A simple calculation also shows that for each $j=1,\dots, K -1$
	\[
	x^{j,\delta}_1 - x^{j+1,\delta}_1 = \frac{a}{(j+2-\kappa)(j+1-\kappa)} \geq \frac{a}{(K+1-\kappa)(K-\kappa)} \geq [2(K+1-\kappa)(K-\kappa)]^{-1}.
	\]
	On the other hand, once again employing the result that $C^{2}t^{2}\geq (4\ceil{t}+3)(4\ceil{t}+4)$, for all $t\geq 1$, (which is a consequence of \eqref{eq:ProbNNSecondCBound}) with $t = C((r\vee s)/\mathrm{p})^2$ we obtain 
	\[
	2\varepsilon((r\vee s)/\mathrm{p}) =  2 C^{-2} ( C((r\vee s)/\mathrm{p})^2)^{-2}\leq  2[(4K+3)(4K+4)]^{-1} <  [2(K+1 - \kappa)(K -\kappa)]^{-1}.
	\] We therefore conclude that $\alpha_{2j-1} > \alpha_{2j} = x^{j,\delta}_1 -\varepsilon((r\vee s)/\mathrm{p}) >  x^{j+1,\delta}_1 +\varepsilon((r\vee s)/\mathrm{p}) = \alpha_{2j+1}$.
	and thus the conditions to apply Lemma \ref{lemma:NNExistence} are met.
	
	Now, let $\psi$ be the network provided by Lemma \ref{lemma:NNExistence} with this sequence $\{\alpha_j\}_{j=1}^{2K}$. Because of the definition of $\alpha_j$ and the conclusion of Lemma \ref{lemma:NNExistence} we have
	\begin{equation*}
		\psi(x) = \begin{cases}
			0 & \text{ if } x_1 \in [x^{k,\delta}_1 - \varepsilon((r\vee s)/\mathrm{p}), x^{k,\delta}_1 + \varepsilon((r\vee s)/\mathrm{p})] \text{ and } k \text{ is odd} \\
			1 & \text{ if } x_1 \in [x^{k,\delta}_1 - \varepsilon((r\vee s)/\mathrm{p}), x^{k,\delta}_1 + \varepsilon((r\vee s)/\mathrm{p})] \text{ and } k \text{ is even}  
		\end{cases}
	\end{equation*}
	Moreover, because of Lemma \ref{lem:NN-separation-lemma}, the fact that the value of $f_a(x)$ depends only on $x_1$ and the bound \eqref{eq:EpsBoundByEpsPrime}
	\begin{equation*}
		f_a(x) = \begin{cases}
			0 & \text{ if } x_1 \in [x^{k,\delta}_1 - \varepsilon((r\vee s)/\mathrm{p}), x^{k,\delta}_1 + \varepsilon((r\vee s)/\mathrm{p})] \text{ and } k \text{ is odd} \\
			1 & \text{ if } x_1 \in [x^{k,\delta}_1 - \varepsilon((r\vee s)/\mathrm{p}), x^{k,\delta}_1 + \varepsilon((r\vee s)/\mathrm{p})] \text{ and } k \text{ is even}  
		\end{cases}
	\end{equation*}
	In particular, $\psi(x) = f_a(x)$ whenever $x_1 \in [x^{k,\delta}_1 - \varepsilon((r\vee s)/\mathrm{p}), x^{k,\delta}_1 + \varepsilon((r\vee s)/\mathrm{p})]$ for any $k \in \{1,2,\dotsc,K\}$. 
	
	To see that $\psi(x) = f_a(x)$ for all $x \in \mathcal{B}_{\varepsilon((r\vee s)/\mathrm{p})}^{\infty}(\mathcal{T} \cup \mathcal{V})$, note that, for every $x \in \mathcal{B}_{\varepsilon((r\vee s)/\mathrm{p})}^{\infty}(\mathcal{T} \cup \mathcal{V})$, there exists an $x^{k,\delta}\in\mathcal{T}\cup\mathcal{V}$ such that $\|x^{k,\delta}-x\|_\infty\leq \varepsilon((r\vee s)/\mathrm{p})$. Then, by the assumption that $\max\{\ell\in\mathbb{N}\,\vert\, x^{\ell,\delta}\in\mathcal{T} \cup \mathcal{V} \}  \leq \ceil{C ((r \vee s)/\mathrm{p})^2} $ occurs in \eqref{eq:NN-useful-event}, we have $k\leq K$, and so $x_1 \in [x^{k,\delta}_1 - \varepsilon((r\vee s)/\mathrm{p}), x^{k,\delta}_1 + \varepsilon((r\vee s)/\mathrm{p})]$. But we have already shown that for such $x$, $\psi(x) = f_a(x)$. Thus the proof of the theorem is complete.
\end{proof}

\subsection{Tools from the Solvability Complexity Index (SCI) hierarchy used for Theorem \ref{theorem:NNNonComp}}\label{Sec:SCIBackground}

In order to formalise the non computability result stated in Theorem \ref{theorem:NNNonComp} we shall summarise appropriate definitions and ideas on the `SCI hierarchy'  \cite{comp_stable_NN22, SCI, Hansen_JAMS, CRAS, Ben-Artzi_FoCM2021, Doyle_McMullen, McMullen1, McMullen2}. The material in this section very closely follows the definitions and presentation in \cite{opt_big} with slight adaptations made owing to the different focus of this paper. Working with the SCI hierarchy and general algorithms allows us to show the non-computability is independent of both the underlying computational model (e.g. a Turing machine, BSS machine) and local minima as in Remark \ref{rem:LocalMinima}. 

It also allows us to easily make non-computability statements applicable to both deterministic and randomised algorithms. We include the ensuing discussion to ensure that this paper is self contained.

\subsubsection{Computational problems}\label{sec:SCI_hierarchy}
We start by defining a \emph{computational problem} \cite{SCI}:

\begin{definition}[Computational problem]\label{definition:ComputationalProblem}
	Let $\Omega$ be some set, which we call the \emph{input} set,
	and $\Lambda$ be a set of complex valued functions on $\Omega$ such that for $\iota_1, \iota_2 \in \Omega$, then $\iota_1 = \iota_2$ if and only if $f(\iota_1) = f(\iota_2)$ for all $f \in \Lambda$. We call $\Lambda$ an \emph{evaluation} set. Let $(\mathcal{M},d_{\mathcal{M}})$ be a metric space, and finally let $\Xi:\Omega\to \mathcal{M}$ be a function which we call the \emph{solution map}.
	We call the collection $\{\Xi,\Omega,\mathcal{M},\Lambda\}$ a \emph{computational problem}.
\end{definition}

The set $\Omega$ is essentially the set of objects that give rise to the various instances of our computational problem. The solution map $\Xi : \Omega\to\mathcal{M}$ is what we are interested in computing. Finally, the set $\Lambda$ is the collection of functions that provide us with the information we are allowed to read. As a simple example, if we were considering matrix inversion then $\Omega$ might be a collection of invertible matrices, $\Xi$ would be the matrix inversion map taking $\Omega$ to the set of matrices and $\Lambda$ would consist of functions that allow us to access entries of the input matrices.

In the slightly more complicated context of a computational problem, the neural network problem formulated in Section \ref{sec:NNComp} can be understood as per the following:
\begin{definition}[Neural network computational problem]Fix $d,r \in \mathbb{N}$, a classification function $f: \real^{d} \to \{0,1\}$, neural network layers and dimensions $L$ and $\mathbf{N} = (N_L=1,N_{L-1},\dotsc, \allowbreak N_1,N_0=d)$ respectively as well as $\epsilon,\hat \epsilon$ and a cost function $\mathcal{R} \in \nncfeps$. The \emph{neural network computational problem}  \[\{\Xi\nnsubs,\Omega\nnsubs,\mathcal{M}\nnsubs,\Lambda\nnsubs\}\] is defined as follows: \begin{enumerate}
		\item The input set $\Omega\nnsubs$ is the collection of all $\mathcal{T}$ with $\mathcal{T}=\{x^1,\hdots, x^r\}$ a finite subset of $\real^d$ such that $\mathcal{T}\in \mathcal{S}^{f}_{\varepsilon'(K)}$ with  $\varepsilon'(n) := [(4n+3)(4n+4)]^{-1}$.
		\item The metric space $\mathcal{M}\nnsubs$ is set to $\mathbb{R}^r$ with the distance function induced by $\|\cdot\|_{*}$ where  $* = 1,2 \text{ or } \infty$ as per the statement of Theorem \ref{theorem:NNNonComp}.
		\item The solution map $\Xi\nnsubs$ is given by the following: for a training set $\mathcal{T}$, we let \[ \mathcal{A}_\mathcal{T}^{\epsilon}:=
		\argmineps  \mathcal{R} \left(\{\varphi({x}^j)\}_{j=1}^r,\{f({x}^j)\}_{j=1}^r\right), 
		\] and then $\Xi\nnsubs(\mathcal{T}) = \{\phi(x^i)\}_{i=1}^{r}$ for $\phi \in \mathcal{A}^{\epsilon}_{\mathcal{T}}$. Note that $\Xi$ is potentially multivalued if $\mathcal{A}^{\epsilon}_{\mathcal{T}}$ has more than one element - this will not be a problem for our theory and will be explained further in Remark \ref{rem:MultivaluedFunctions}.
		\item The set $\Lambda\nnsubs$ is given by 
		\begin{equation}\label{eq:theNNLambda}
			\Lambda\nnsubs=\{f^{j,k}\}_{j=1,k=1}^{j=d,k=r},
		\end{equation}
		where $f^{j,k}(\mathcal{T})=x_j^k$  gives access to the $j$th coordinate of the $k$th vector of the training set.
	\end{enumerate}
\end{definition} 
To reduce the burden on notation, we will abbreviate 
\begin{equation*}
	\{\Xi^\mathcal{NN},\Omega^\mathcal{NN},\mathcal{M}^\mathcal{NN},\Lambda^\mathcal{NN}\}=	\{\Xi\nnsubs,\Omega\nnsubs,\mathcal{M}\nnsubs,\Lambda\nnsubs\}
\end{equation*}
where there is no ambiguity surrounding the parameters $f,r,\epsilon,\mathcal{R},\mathbf{N},L$.

\begin{remark}[Existence of a neural network]\label{rem:NNCompExistence}
	It may not be a-priori obvious that the set $\mathcal{A}_\mathcal{T}^{\epsilon}$ is non-empty and thus $\Xi\nnsubs(\mathcal{T})$ is well defined. In fact, this is an immediate consequence the fact that the cost function $\mathcal{R}$ is a member of 
	$\mathcal{CF}^{\epsilon,\hat \epsilon}_{r}$ defined in \eqref{eq:CFDEps} and the definition of $\operatorname {argmin}_{\epsilon}$ given in \eqref{eq:argminEpsDef}. In particular, the existence of an approximate minimiser is guaranteed since $\mathcal{R}$ is bounded from below.
\end{remark}

\subsubsection{Algorithms}
In this section we shall describe the algorithms that are designed to approximate the solution map $\Xi$ in a computational problem $\{\Xi,\Omega,\mathcal{M},\Lambda\}$. We shall start with deterministic general algorithms:
\begin{definition}[General Algorithm]\label{definition:Algorithm}
	Given a  computational problem $\{\Xi,\Omega,\mathcal{M},\Lambda\}$, a \emph{general algorithm} is a mapping $\Gamma:\Omega\to\mathcal{M}\cup \{\nh\}$ such that, for every $\iota\in\Omega$, the following conditions hold:
	\begin{enumerate}[label=(\roman*)]
		\item there exists a nonempty subset of evaluations $\Lambda_\Gamma(\iota) \subset\Lambda $, and, whenever $\Gamma(\iota) \neq  \nh$, we have $|\Lambda_\Gamma(\iota)|<\infty$\label{property:AlgorithmFiniteInput},
		\item  the action of $\,\Gamma$ on $\iota$ is uniquely determined by $\{f(\iota)\}_{f \in \Lambda_\Gamma(\iota)}$, \label{property:AlgorithmDependenceOnInput}
		\item for every $\iota^{\prime} \in\Omega$ such that $f(\iota^\prime)=f(\iota)$ for all $f\in\Lambda_\Gamma(\iota)$, it holds that $\Lambda_\Gamma(\iota^{\prime})=\Lambda_\Gamma(\iota)$.\label{property:AlgorithmSameInputSameInputTaken}
	\end{enumerate}
\end{definition}

\begin{remark}[The purpose of a general algorithm: universal impossibility results]
	The purpose of a general algorithm is to have a definition that will encompass any model of computation and that will allow impossibility results to become universal. Given that there are several non-equivalent models of computation, impossibility results will be shown with this general definition of an algorithm.
\end{remark}
\begin{remark}[The power of a general algorithm] \label{rem:GAPower}
	General algorithms are extremely powerful computational models with every Turing or BSS machine a general algorithm but the converse does not hold. Thus a non-computability result proven using general algorithms is strictly stronger than one proven only for Turing machines or BSS machines.
	
	In particular, general algorithms are more powerful than any Turing machine or BSS machine, or even such a machine with access to an oracle that provides an approximate minimiser 
		\[
		\phi \in \mathop{\mathrm{arg min}_{\epsilon}}_{\tilde\phi \in \mathcal{NN}_{\mathbf{N},L}}  \mathcal{R}\left(\{\tilde\phi({x}^j)\}_{j=1}^r,\{f({x}^j)\}_{j=1}^r \right)
		\]
	for every inexact input provided to the algorithm, or an oracle that detects when an algorithm has encountered local minima. It is for this reason that we stated in Remark \ref{rem:LocalMinima} that local minima were not relevant to Theorem \ref{theorem:NNNonComp}.
\end{remark}
\begin{remark}[The non-halting output $\nh$]
	The non-halting ``output'' $\nh$ of a general algorithm may seem like an unnecessary distraction given that a general algorithm is just a mapping, which is strictly more powerful than a Turing or a BSS machine. However, the $\nh$ output is needed when the concept of a general algorithm is extended to a randomised general algorithm.
	A technical remark about $\nh$ is also appropriate, namely that $\Lambda_{\Gamma}(\iota)$ is allowed to be infinite in the case when $\Gamma(\iota) = \nh$. This is to allow general algorithms to capture the behaviour of a Turing or a BSS machine not halting by virtue of requiring an infinite amount of input information.
\end{remark}

Owing to the presence of the special non-halting ``output'' $\nh$, we have to extend the metric $d_{\mathcal{M}}$ on $\mathcal{M}\times \mathcal{M}$ to $d_{\mathcal{M}}:\mathcal{M} \cup \{\nh\} \times \mathcal{M} \cup \{\nh\}\to \real_{\geq 0}$ in the following way:
\begin{equation}\label{eq:extended-metric}
	d_{\mathcal{M}}(x,y) = \begin{cases} d_{\mathcal{M}}(x,y) & \text{ if } x,y \in \mathcal{M} \\
		0 & \text{ if } x = y = \nh\\
		\infty & \text{ otherwise.} \end{cases}
\end{equation}

Definition \ref{definition:Algorithm} is sufficient for defining a randomised general algorithm, which is the only tool from the SCI theory needed in order to prove Theorem \ref{theorem:NNNonComp}.

\begin{remark}[Multivalued functions]\label{rem:MultivaluedFunctions} When dealing with optimisation problems one needs a framework that can handle multiple solutions. As the setup above does not allow $\Xi$ to be multi-valued we need some slight changes.  We allow $\Xi$ to be multivalued, even though a general algorithm is assumed not to be. For $\iota \in \Omega$, we define 
	$
	\disM(\Xi(\iota),\Gamma(\iota)) := \inf_{x \in \Xi(\iota)}d_{\mathcal{M}}(x,\Gamma(\iota)).
	$ That is to say, the error that $\Gamma$ is assumed to incur in trying to compute $\Xi(\iota)$ is the best (infimum) of all possible errors across all values of $\Xi(\iota)$.
\end{remark}
One final definition that is useful  is that of the \emph{minimum amount of input information}, defined if $\Lambda$ is countable. Although this definition has its own uses in other work on the SCI Hierarchy, in the context of this paper it will only be useful to address a technicality in the next section.

\begin{definition}[Minimum amount of input information]\label{def:minimum_runtime}
	Given the computational problem $\{\Xi,\Omega,\mathcal{M},\allowbreak\Lambda\}$, where $\Lambda =\{f_k \, \vert \, k \in \mathbb{N}, \, k\leq |\Lambda| \}$ and a general algorithm $\Gamma$, we define the \emph{minimum amount of input information} $T_{\Gamma}(\iota)$ for $\Gamma$ and $\iota \in \Omega$ as 
	\[
	T_{\Gamma}(\iota) :=\sup\lbrace m \in \mathbb{N} \, \vert \, f_{m} \in \Lambda_{\Gamma}(\iota) \rbrace.
	\]
	Note that, for $\iota$ such that $\Gamma(\iota) = \nh$, the set $\Lambda_{\Gamma}(\iota)$ may be infinite (see Definition \ref{definition:Algorithm}), in which case $T_{\Gamma}(\iota)=\infty$.
	
\end{definition}
\subsubsection{Randomised algorithms}

In many contemporary fields of mathematics of information such as deep learning, the use of randomised algorithms is widespread. We therefore need to extend the concept of a general algorithm to a \emph{randomised random algorithm}.

\begin{definition}[Randomised General Algorithm]\label{definition:ProbablisticAlgorithm}	Given a computational problem $\{\Xi,\Omega,\mathcal{M},\Lambda\}$, where $\Lambda = \{f_k \, \vert \, k \in \mathbb{N}, \, k\leq |\Lambda| \}$, a \emph{randomised general algorithm} (RGA) is a collection $X$ of general algorithms $\Gamma:\Omega\to\mathcal{M}\cup \{\nh\}$, a sigma-algebra $\mathcal{F}$ on $X$, and a family of probability measures $\{\mathbb{P}_{\iota}\}_{\iota \in \Omega}$ on $\mathcal{F}$ such that the following conditions hold:
	\begin{enumerate}[label=(P\roman*)]
		\item For each $\iota \in \Omega$, the mapping $\gprob_{\iota}:(X,\mathcal{F}) \to (\mathcal{M}\cup \{\nh\}, \mathcal{B})$ defined by $\gprob_{\iota}(\Gamma) = \Gamma(\iota)$ is a random variable, where $\mathcal{B}$ is the Borel sigma-algebra on $\mathcal{M}\cup \{\nh\}$. \label{property:PAlgorithmMeasurable}
		\item For each $n \in \mathbb{N}$ and $\iota \in \Omega$, we have $\lbrace \Gamma \in X \, \vert \, T_{\Gamma}(\iota) \leq n \rbrace \in \mathcal{F}$.
		\label{property:PAlgorithmRTMeasurable}
		
		\item For all $\iota_1,\iota_2 \in \Omega$ and $E \in \mathcal{F}$ so that, for every $\Gamma \in E$ and every $f \in \Lambda_{\Gamma}(\iota_1)$, we have $f(\iota_1) = f(\iota_2)$, it holds that $\mathbb{P}_{\iota_1}(E) = \mathbb{P}_{\iota_2}(E)$.
		\label{property:PAlgorithmConsistent}
	\end{enumerate}
	It is not immediately clear whether condition \ref{property:PAlgorithmRTMeasurable} for a given RGA $(X,\mathcal{F},\{\pr_\iota\}_{\iota\in\Omega})$ holds independently of the choice of the enumeration of $\Lambda$. This is indeed the case, but we shall not show this here (see \cite{opt_big}
	for further information).
\end{definition}

\begin{remark}[Assumption \ref{property:PAlgorithmRTMeasurable}]
	Note that \ref{property:PAlgorithmRTMeasurable} in Definition \ref{definition:ProbablisticAlgorithm} is needed in order to ensure that the minimum amount of input information (that is, the amount of input information the algorithm makes use of) also becomes a valid random variable. More specifically, for each $\iota \in \Omega$, we define the random variable
	\[
	T_{\gprob}(\iota): X\to\mathbb{N}\cup\{\infty\} \text{ according to } \Gamma \mapsto T_{\Gamma}(\iota).
	\]
	Assumption \ref{property:PAlgorithmRTMeasurable} ensures that this is indeed a random variable. 
	
	As the minimum amount of input information is typically related to the complexity of an algorithm, one would be dealing with a rather exotic probabilistic model if $T_{\gprob}(\iota)$ were not a random variable. Indeed, note that the standard models of randomised algorithms (see \cite{Arora2007}) can be considered as RGAs (in particular, they will satisfy  \ref{property:PAlgorithmRTMeasurable}).
\end{remark}	

\begin{remark}[The purpose of a randomised general algorithm: universal lower bounds]
	As for a general algorithm, the purpose of a randomised general algorithm is to have a definition that will encompass every model of computation, which will allow lower bounds and impossibility results to be universal. Indeed, randomised Turing and BSS machines can be viewed as randomised general algorithms.
\end{remark}

We will, with a slight abuse of notation, also write $\mathrm{RGA}$ for the family of all randomised general algorithms for a given a computational problem and refer to the algorithms in $\mathrm{RGA}$ by  $\gprob$. 		
With the definitions above we can now make probabilistic version of the strong breakdown epsilon as follows.

\begin{definition}[Probabilistic strong breakdown epsilon]\label{prob_strong_break} Given a computational problem $\{\Xi,\Omega,\mathcal{M},\Lambda\}$, where $\Lambda = \{f_k \, \vert \, k \in \mathbb{N}, \, k\leq |\Lambda| \}$, 
	we define the \emph{probabilistic strong breakdown epsilon} $\epsilon_{\mathbb{P}\mathrm{B}}^{\mathrm{s}}: [0,1) \to \mathbb{R}$ according to
	\begin{align*}
		\epsilon_{\mathbb{P}\mathrm{B}}^{\mathrm{s}}(\mathrm{p}) = \sup\{&\epsilon \geq 0, \, \vert \, \forall \, \gprob \in \mathrm{RGA} \,\,\exists \, \iota \in \Omega \text{ such that } \mathbb{P}_{\iota}(\disM(\gprob_{\iota},\Xi(\iota)) > \epsilon) > \mathrm{p}\},
	\end{align*}
	where $\gprob_{\iota}$ is defined in \ref{property:PAlgorithmMeasurable} in Definition \ref{definition:ProbablisticAlgorithm}. 
\end{definition}

Note that the probabilistic strong Breakdown-epsilon is not a single number but a function of $\mathrm{p}$. Specifically, it is the largest $\epsilon$ so that the probability of failure with at least $\epsilon$-error is greater than $\mathrm{p}$.

\subsubsection{Inexact input and perturbations} \label{sec:inexact}
Suppose we are given a computational problem $\{\Xi, \Omega, \mathcal{M}, \Lambda\}$, and that 
$
\Lambda = \{f_j\}_{j \in \beta},
$
where $\beta$ is some index set that can be finite or infinite.  Obtaining $f_j$ may be a computational task on its own, which is exactly the problem in most areas of 
computational mathematics. In particular, for $\iota \in \Omega$, $f_j(\iota)$ could be the number $e^{\frac{\pi}{j} i }$ for example. Hence, we cannot access $f_j(\iota)$, but rather $f_{j,n}(\iota)$ where $f_{j,n}(\iota) \rightarrow f_{j}(\iota)$ as $n \rightarrow \infty$. In this paper we will be interested in the case when this can be done with error control. In particular, we consider $f_{j,n}: \Omega \to \dyadic_n + i \dyadic_n$, where $\dyadic_n:=\{k\,2^{-n}\, \vert \,k\in\mathbb{Z}\}$, such that
\begin{equation}\label{Lambda_limits2}
	\|\{f_{j,n}(\iota)\}_{j\in\beta} - \{f_j(\iota)\}_{j\in\beta}\|_{\infty} \leq 2^{-n}, \quad \forall \iota \in \Omega.
\end{equation}
We will call a collection of such functions $\Delta_1$-information for the computational problem. Formally, we have the following.

\begin{definition}[$\Delta_{1}$-information]\label{definition:Lambda_limits}
	Let $\{\Xi, \Omega, \mathcal{M}, \Lambda\}$ be a computational problem with $\Lambda = \{f_j\}_{j \in \beta}$. Suppose that, for each $j\in\beta$ and $n\in\mathbb{N}$, there exists an  $f_{j,n}: \Omega \to \dyadic_n + i \dyadic_n$ such that \eqref{Lambda_limits2} holds. We then say that the set $\hat\Lambda=\{f_{j,n} \, \vert \, j\in\beta ,n\in\mathbb{N}\}$ provides $\Delta_1$-information for $\{\Xi, \Omega, \mathcal{M}, \Lambda\}$. 
\end{definition}

We can now define what we mean by a computational problem with $\Delta_1$-information.

\begin{definition}[Computational problem with $\Delta_1$-information]\label{definition:Omega_tilde_Delta_m}
	Given $\{\Xi,\Omega,\mathcal{M},\Lambda\}$ with  $\Lambda=\{f_j\}_{j \in \beta}$, the corresponding \emph{computational problem with $\Delta_1$-information} is defined as 
	\[
	\{\Xi,\Omega,\mathcal{M},\Lambda\}^{\Delta_1} := \{\tilde \Xi,\tilde \Omega,\mathcal{M},\tilde \Lambda\},
	\]
	where 
	\begin{equation}\label{eq:omega_tilde}
		\tilde \Omega = \left\{ \tilde \iota = \big\{(f_{j,1}(\iota), f_{j,2}(\iota), f_{j,3}(\iota), \dots) \big\}_{j \in \beta} \, \vert \, \iota \in \Omega,\; f_{j,n}: \Omega \to \dyadic_n + i \dyadic_n \text{ satisfy \eqref{Lambda_limits2}} \right\},
	\end{equation}
	$\tilde \Xi(\tilde \iota) = \Xi(\iota)$, and $\tilde \Lambda = \{\tilde f_{j,n}\}_{j,n \in \beta \times \mathbb{N}}$, where $\tilde f_{j,n}(\tilde \iota) = f_{j,n}(\iota)$. Given an $\tilde\iota\in\tilde\Omega$, there is a unique $\iota\in\Omega$ for which $\tilde\iota=\big\{(f_{j,1}(\iota), f_{j,2}(\iota), f_{j,3}(\iota), \dots) \big\}_{j \in \beta}$ (by Definition \ref{definition:ComputationalProblem}). We say that this $\iota\in\Omega$ \emph{corresponds} to $\tilde\iota\in\tilde\Omega$.
\end{definition}
\begin{remark}
	Note that the correspondence of a unique $\iota$ to each $\tilde\iota$ in Definition \ref{definition:Omega_tilde_Delta_m} ensures that $\tilde\Xi$ and the elements of $\tilde\Lambda$ are well-defined.
\end{remark}

One may interpret the computational problem $\{\Xi,\Omega,\mathcal{M},\Lambda\}^{\Delta_1} = \{\tilde \Xi,\tilde \Omega,\mathcal{M},\tilde \Lambda\}$ as follows. The collection $\tilde \Omega$ is the family of all sequences approximating the inputs in $\Omega$. For an algorithm to be successful for $\{\Xi,\Omega,\mathcal{M},\Lambda\}^{\Delta_1}$ it must work for all $\tilde \iota \in \tilde \Omega$, that is, for any sequence approximating $\iota$.

\begin{remark}[Oracle tape/node providing $\Delta_1$-information]\label{rem:oracles}
	For impossibility results we use general algorithms and randomised general algorithms (as defined below), and thus, due to their generality, we do not need to specify how the algorithms read the information. 
\end{remark}

The next proposition serves as the key building block for Theorem \ref{theorem:NNNonComp} and is proven in \cite[Prop 9.5]{opt_big}. Note that the proposition is about arbitrary computational problems, and is hence also a tool for demonstrating lower bounds on the breakdown epsilon for general computational problems. 
\begin{proposition}\label{prop:DrivingNegativeProposition}
	Let $\{\Xi, \Omega, \mathcal{M}, \Lambda\}$ be a computational problem with $\Lambda=\{f_k\,\vert\,k\in\mathbb{N}, k\leq|\Lambda| \}$ countable. Suppose that $\iota^0 \in \Omega$ and that $\{\iota^1_n\}_{n=1}^{\infty}$ is a sequence in $\Omega$ so that the following conditions hold:
	\begin{enumerate}[label=(P\alph*)]
		\item For every $k\leq |\Lambda|$ and for all $n \in\mathbb{N}$ we have $|f_k(\iota^j_n) - f_k(\iota^0)|\leq 1/4^n$. \label{property:Del1Info}
		\item There is a $\kappa > 0$ such that
		$
		\inf_{\upsilon^1 \in \Xi(\iota^1_n), \upsilon^2 \in \Xi(\iota^0)}d_{\mathcal{M}}(\upsilon^1,\upsilon^2) \geq \kappa
		$. \label{property:MinimisersOfNonZero}
	\end{enumerate}
	Then the computational problem $\{\Xi,\Omega,\mathcal{M},\Lambda\}^{\Delta_1}$ satisfies $\strbdepsp \geq \kappa/2$ for $\mathrm{p} \in [0,1/2)$. 
\end{proposition}		

\subsection{Stating Theorem \ref{theorem:NNNonComp} in the SCI language - Proposition \ref{prop:NNNonComp-formal}}
A slightly stronger formal statement of Theorem \ref{theorem:NNNonComp} in the SCI language is now as follows.
\begin{proposition}\label{prop:NNNonComp-formal}
	There is an uncountable collection $\mathcal{C}_1$ of classification functions $f$ as in \eqref{eq:the_fs} -- with fixed $d \geq 2$ -- such that for
	\begin{enumerate}
		\item any neural  network dimensions $\mathbf{N} = (N_L=1,N_{L-1},\dotsc, \allowbreak N_1,N_0=d)$ with $L \geq 2$,
		\item any  $r \geq 3(N_1+1) \dotsb (N_{L-1}+1)$,
		\item and any $\epsilon > 0$, $\hat \epsilon \in (0,1/2)$ and cost function $\mathcal{R} \in \nncf^{\epsilon,\hat\epsilon}_r$.
	\end{enumerate}
	There is an uncountable collection $\mathcal{C}_3$ of disjoint subsets of $\Omega\nnsubs$ so that for each $\hat \Omega \in \mathcal{C}_3$ the computational problem \[\{\Xi\nnsubs,\hat\Omega,\mathcal{M}\nnsubs,\Lambda\nnsubs\}^{\Delta_1}\] has breakdown epsilon $\strbdepsp \geq 1/4 - \hat{\epsilon}/2$, for all $\mathrm{p} \in [0,1/2)$.
	
\end{proposition}

To see that Proposition \ref{prop:NNNonComp-formal} implies Theorem \ref{theorem:NNNonComp}, assume that Proposition \ref{prop:NNNonComp-formal} holds and suppose that $\Gamma$ is a randomised algorithm (in either the BSS or Turing models). The existence of $\phi$ stated in Theorem \ref{theorem:NNNonComp} is guaranteed as per the discussion in Remark \ref{rem:NNCompExistence}. Furthermore $\Gamma$ is also a randomised general algorithm and hence we can consider $\Gamma$ restricted to $\hat{\Omega}$ for each $\hat{\Omega} \in \mathcal{C}_3$. Since the computational problem $\{\Xi\nnsubs,\hat\Omega,\mathcal{M}\nnsubs,\Lambda\nnsubs\}^{\Delta_1}$ has $\strbdepsp \geq 1/4 - \hat{\epsilon}/2 > 1/4 - 3\hat{\epsilon}/4$, for all $\mathrm{p} \in [0,1/2)$ there must exist a training set $\mathcal{T} = \mathcal{T}(\hat \Omega)$ with $\mathcal{T} = \{x^1,x^2,\dotsc,x^r\}$ for which 
\[
\mathbb{P}\Big(\|\{\Gamma_{\mathcal{T}}(x^j)\}_{j=1}^r - \{\phi(x^j)\}_{j=1}^r\|_{*} \geq 1/4-3\hat{\epsilon}/4\Big ) > \mathrm{p},
\]
for any 	$
\phi \in \argmineps  \mathcal{R} \left(\{\varphi({x}^j)\}_{j=1}^r,\{f({x}^j)\}_{j=1}^r\right)
$ (this is itself a consequence of Remark \ref{rem:MultivaluedFunctions}).

We now choose $\mathcal{C}_2 = \{\mathcal{T}(\hat \Omega) \, \vert \, \hat \Omega \in \mathcal{C}_3\}$. Because $\mathcal{C}_3$ is an uncountable collection of disjoint sets, $\mathcal{C}_2$ is uncountable and thus Theorem \ref{theorem:NNNonComp} follows.

\subsection{Proof of Proposition \ref{prop:NNNonComp-formal} and Theorem \ref{theorem:NNNonComp}}
As demonstrated in the previous section, to prove Theorem \ref{theorem:NNNonComp} it suffices to prove Proposition \ref{prop:NNNonComp-formal}.
We begin by starting the following useful lemma:
\begin{lemma}\label{lem:NNSuccess}
	Recall the setup of Proposition \ref{prop:NNNonComp-formal} and the vectors $x^{k,\delta}$ defined in \eqref{eq:theNNexes}. For any $\delta\in  (0, \varepsilon'(r))$ and arbitrary
	\begin{align} 
		\phi &\in  \argmineps  \mathcal{R}\left(\{\varphi(x^{j,\delta})\}_{j=1}^r ,\{f_a(x^{j,\delta})\}_{j=1}^r\right)\label{eq:phi_del}
	\end{align}
	we have $|\phi(x^{k,\delta}) - f_a(x^{k,\delta})| \leq \hat \epsilon$ for all $k \in \{1,\dotsc, r\}$.
\end{lemma}
\begin{proof}
	By Lemma \ref{lem:NNUnstableExists}, there exists a neural network $\tilde \varphi \in \mathcal{NN}_{\mathbf{N},L}$ with $\tilde \varphi(x^{k,\delta}) = f_a(x^{k,\delta})$ for all $k$. In particular, $ \mathcal{R} \left(\{\tilde \varphi ({x}^{j,\delta})\}_{j=1}^r,\{f({x}^{j,\delta})\}_{j=1}^r\right) = 0$. Thus, by \eqref{eq:phi_del} and the definition of the approximate argmin as in \eqref{eq:argminEpsDef}, we must have that \[ \mathcal{R} \left(\{ \phi ({x}^{j,\delta})\}_{j=1}^r,\{f({x}^{j,\delta})\}_{j=1}^r\right) \leq \epsilon\] and the conclusion of the claim follows because $\mathcal{R}\in \nncfeps$ as defined in \eqref{eq:CFDEps}.
	
\end{proof}
Now that we have proven Lemma \ref{lem:NNSuccess}, we are ready to prove Proposition \ref{prop:NNNonComp-formal}.
\begin{proof}[Proof of Proposition \ref{prop:NNNonComp-formal}]
	
	As in the proof of Theorem \ref{theorem:NNCV}, we begin by defining the sets $\mathcal{C}_1$ and $\mathcal{C}_3$. Let $\mathcal{C}_1 = \{ f_a: \real^d \to [0,1] \, \vert \, a \in [1/2,1]\}$, where $f_a$ is defined as in \eqref{eq:TheFs}. Fix $a \in [1/2,1]$ and $\kappa \in [1/4,3/4]$, define $\mathcal{T}^{\kappa}_{\delta}:= \{x^{1,\delta},x^{2,\delta},\dotsc,x^{r,\delta}\}$ where the values $x^{i,\delta}$ (each depending on $\kappa$ and $a$) are defined in \eqref{eq:theNNexes}. We define
	$\hat \Omega^{\kappa}:= \{\mathcal{T}^{\kappa}_{\delta} \, \vert \, \delta \in [0,\epsilon'(r))$. By Lemma \ref{lem:NN-separation-lemma}, we have $\mathcal{T}^{\kappa}_{\delta} \in \mathcal{S}^{f_a}_{\varepsilon'(r)}$ so that $\hat \Omega^{\kappa} \subset \Omega^{\mathcal{NN}}$. Note also that noting that the $\hat \Omega^{\kappa}$ are disjoint as an immediate consequence of \eqref{eq:theNNexes}. Finally, we set $\mathcal{C}_3 := \{ \hat \Omega^{\kappa} \, \vert \, \kappa \in [1/4,3/4]\}$, .
	
	Now we have defined $\mathcal{C}_1$ and $\mathcal{C}_3$, we will show that for any $\kappa \in [1/4,3/4]$ the computational problem \[\{\Xi^\mathcal{NN},\hat\Omega^\kappa,\mathcal{M}^\mathcal{NN},\Lambda^\mathcal{NN}\}^{\Delta_1}\] has breakdown epsilon $\strbdepsph \geq 1/4 - \hat{\epsilon}/2$, for all $\mathrm{p} \in [0,1/2)$. This will be done using Proposition \ref{prop:DrivingNegativeProposition}. We will define $\iota^0:= \mathcal{T}^{\kappa}_{0}$ and $\iota^1_n:= \mathcal{T}^{\kappa}_{4^{-n}}$. 
	
	By \eqref{eq:theNNexes}, we see that  $\|x^{\,j,4^{-n}}-x^{j,0}\|_\infty\leq 4^{-n}$ for $j=1,2,\dotsc,r$. Hence (recalling the definition of $\Lambda^\mathcal{NN}$), property \ref{property:Del1Info} from Proposition \ref{prop:DrivingNegativeProposition} holds.

	Fix $n \in \mathbb{N}$ sufficiently large and let $\phi_0$ and $\phi_n$ be arbitrary neural networks so that
	\begin{equation}\label{eq:CompPhiChoice}
	\begin{split}
		&\phi_0 \in \argmineps \left(\{\varphi(x^{j,0})\}_{j=1}^r ,\{f_a(x^{j,0})\}_{j=1}^r\right)\\
		 &\phi_n \in \argmineps \left(\{\varphi(x^{j,4^{-n}})\}_{j=1}^r ,\{f_a(x^{j,4^{-n}})\}_{j=1}^r\right).
         \end{split}
         \end{equation} By Lemma \ref{lem:NoNN} and the assumption that $|\mathcal{T}^{\kappa}_0| = r \geq 3(N_1+1) \dotsb (N_{L-1}+1)$ we conclude that 
	\[\max\limits_{j=1,2,\dotsc,r} |\phi_0(x^{j,0}) - f_a(x^{j,0})| \geq 1/2. \] By contrast, Lemma \ref{lem:NNSuccess} shows that $\max_{j=1,2,\dotsc,r} |\phi_n(x^{j,4^{-n}}) - f_a(x^{j,4^{-n}})| \leq \hat \epsilon$. Combining these two results and the fact that $f_a(x^{j,0}) = f_a(x^{j,4^{-n}})$ for each $j=1,2,\dotsc,r$ yields
	\begin{equation*}
		\max_{j=1,2,\dotsc,r} |\phi_0(x^{j,0}) - \phi_n(x^{j,4^{-n}})| \geq 1/2  - \hat \epsilon.
	\end{equation*}

	Therefore, since both the $\ell^1$ and $\ell^2$ norms are bounded from below by the $\ell^{\infty}$ norm and $\phi_0$ and $\phi_n$ were chosen arbitrarily according to \eqref{eq:CompPhiChoice}, we have 
	$
	\inf_{\upsilon^1 \in \Xi(\iota^1_n), \upsilon^2 \in \Xi(\iota^0)}d_{\mathcal{M}}(\upsilon^1,\upsilon^2) \geq 1/2 - \hat \epsilon
	$ where $d_{\mathcal{M}}$ is the $\ell^{*}$ norm with $*= 1,2$ or $\infty$.	Hence, property \ref{property:MinimisersOfNonZero} from Proposition \ref{prop:DrivingNegativeProposition} holds with $\kappa = 1/2 - \hat \epsilon$, thereby concluding the proof.
\end{proof}


\bibliographystyle{abbrv}
\bibliography{FoundOpt_AI_Arxiv}

\begin{thebibliography}{10}

\bibitem{adcock2020gap}
B.~Adcock and N.~Dexter.
\newblock The gap between theory and practice in function approximation with
  deep neural networks.
\newblock {\em SIAM Journal on Mathematics of Data Science}, 3(2):624--655,
  2021.

\bibitem{AdcockHansenBook}
B.~Adcock and A.~C. Hansen.
\newblock {\em Compressive Imaging: Structure, Sampling, Learning}.
\newblock Cambridge University Press, 2021.

\bibitem{akhtar2018threat}
N.~Akhtar and A.~Mian.
\newblock Threat of adversarial attacks on deep learning in computer vision:
  {A} survey.
\newblock {\em IEEE Access}, 6:14410--14430, 2018.

\bibitem{antun2020instabilities}
V.~Antun, F.~Renna, C.~Poon, B.~Adcock, and A.~C. Hansen.
\newblock On instabilities of deep learning in image reconstruction and the
  potential costs of {AI}.
\newblock {\em Proc. Natl. Acad. Sci. USA}, 117(48):30088--30095, 2020.

\bibitem{Arora2007}
S.~Arora and B.~Barak.
\newblock {\em Computational Complexity - A Modern Approach}.
\newblock Princeton University Press, 2009.

\bibitem{CRP}
A.~Bastounis, P.~Campodonico, M.~van~der Schaar, B.~Adcock, and A.~C. Hansen.
\newblock On the consistent reasoning paradox of intelligence and optimal trust
  in {AI:} the power of 'i don't know'.
\newblock {\em CoRR}, abs/2408.02357, 2024.

\bibitem{BCH1}
A.~Bastounis, F.~Cucker, and A.~C. Hansen.
\newblock When can you trust feature selection? -- i: {A} condition-based
  analysis of lasso and generalised hardness of approximation.
\newblock {\em Preprint}, 2023.

\bibitem{BGHHPSTZ}
A.~Bastounis, A.~N. Gorban, A.~C. Hansen, D.~J. Higham, D.~Prokhorov,
  O.~Sutton, I.~Y. Tyukin, and Q.~Zhou.
\newblock The boundaries of verifiable accuracy, robustness,
  and generalisation in deep learning.
\newblock In L.~Iliadis, A.~Papaleonidas, P.~Angelov, and C.~Jayne, editors,
  {\em Artificial Neural Networks and Machine Learning -- ICANN 2023}, pages
  530--541, Cham, 2023. Springer Nature Switzerland.

\bibitem{opt_big}
A.~Bastounis, A.~C. Hansen, and V.~{Vla\v{c}i\'{c}}.
\newblock The extended {S}male's 9th problem -- {O}n computational barriers and
  paradoxes in estimation, regularisation, computer-assisted proofs and
  learning.
\newblock {\em Preprint}, 2021.

\bibitem{Des2023}
L.~Beerens and D.~J. Higham.
\newblock Adversarial ink: componentwise backward error attacks on deep
  learning.
\newblock {\em IMA Journal of Applied Mathematics}, 89(1):175--196, 2023.

\bibitem{SCI}
J.~Ben-Artzi, M.~J. Colbrook, A.~C. Hansen, O.~Nevanlinna, and M.~Seidel.
\newblock Computing spectra -- {O}n the solvability complexity index hierarchy
  and towers of algorithms.
\newblock {\em arXiv:1508.03280}, 2020.

\bibitem{CRAS}
J.~Ben-Artzi, A.~C. Hansen, O.~Nevanlinna, and M.~Seidel.
\newblock New barriers in complexity theory: On the solvability complexity
  index and the towers of algorithms.
\newblock {\em Comptes Rendus Mathematique}, 353(10):931 -- 936, 2015.

\bibitem{Ben-Artzi_FoCM2021}
J.~Ben-Artzi, M.~Marletta, and F.~R{\"o}sler.
\newblock Computing the sound of the sea in a seashell.
\newblock {\em Foundations of Computational Mathematics}, 2021.

\bibitem{Nemirovski_robust}
A.~Ben-Tal, L.~El~Ghaoui, and A.~Nemirovski.
\newblock {\em Robust Optimization}.
\newblock Princeton Series in Applied Mathematics. Princeton University Press,
  October 2009.

\bibitem{NemirovskiLMCO}
A.~Ben-Tal and A.~Nemirovski.
\newblock {Lectures on Modern Convex Optimization: Analysis, Algorithms, and
  Engineering Applications}.
\newblock Available online at \url{https://www2.isye.gatech.edu/~nemirovs/},
  2000.

\bibitem{Nemirovski_robust2}
A.~Ben-Tal and A.~Nemirovski.
\newblock Robust solutions of linear programming problems contaminated with
  uncertain data.
\newblock {\em Mathematical Programming}, 88(3):411--424, 2000.

\bibitem{bishop1967foundations}
E.~Bishop.
\newblock {\em Foundations of Constructive Analysis}.
\newblock McGraw-Hill Series in higher mathematics. McGraw-Hill, 1967.

\bibitem{BSS_machine}
L.~Blum, M.~Shub, and S.~Smale.
\newblock {On a theory of computation and complexity over the real numbers:
  $NP$- completeness, recursive functions and universal machines}.
\newblock {\em Bulletin of the American Mathematical Society}, 21(1):1 -- 46,
  1989.

\bibitem{bungert2023trilos_binary}
L.~Bungert, N.~García~Trillos, and R.~Murray.
\newblock {The geometry of adversarial training in binary classification}.
\newblock {\em Information and Inference: A Journal of the IMA},
  12(2):921--968, 01 2023.

\bibitem{carlini2018audio}
N.~Carlini and D.~Wagner.
\newblock Audio adversarial examples: {T}argeted attacks on speech-to-text.
\newblock In {\em 2018 IEEE Security and Privacy Workshops (SPW)}, pages 1--7.
  IEEE, 2018.

\bibitem{Choi_IEEE}
C.~Choi.
\newblock 7 revealing ways {AI}s fail: {N}eural networks can be disastrously
  brittle, forgetful, and surprisingly bad at math.
\newblock {\em IEEE Spectrum}, 58(10):42--47, 2021.

\bibitem{Choi2}
C.~Choi.
\newblock Some {AI} systems may be impossible to compute.
\newblock {\em IEEE Spectrum}, March, 2022.

\bibitem{Colbrook_2019}
M.~Colbrook.
\newblock On the computation of geometric features of spectra of linear
  operators on hilbert spaces.
\newblock {\em Foundations of Computational Mathematics}, 2022.

\bibitem{colbrook_spec_meas}
M.~J. Colbrook.
\newblock Computing spectral measures and spectral types.
\newblock {\em Communications in Mathematical Physics}, 384(1):433--501, 2021.

\bibitem{comp_stable_NN22}
M.~J. Colbrook, V.~Antun, and A.~C. Hansen.
\newblock The difficulty of computing stable and accurate neural networks: On
  the barriers of deep learning and smale's 18th problem.
\newblock {\em Proc.\ Natl.\ Acad.\ Sci.\ USA}, 119(12):e2107151119, 2022.

\bibitem{colbrook2019foundations}
M.~J. Colbrook and A.~C. Hansen.
\newblock {The foundations of spectral computations via the Solvability
  Complexity Index hierarchy}.
\newblock {\em J. Eur. Math. Soc.}, 25(12):4639--4718, 2022.

\bibitem{EU_PressRelease}
E.~Commission.
\newblock Europe fit for the digital age:
  https://digital-strategy.ec.europa.eu/en/news/europe-fit-digital-age-commission-proposes-new-rules-and-actions-excellence-and-trust-artificial.
\newblock {\em Press Release}, April, 2021.

\bibitem{Cucker_Smale97}
F.~Cucker and S.~Smale.
\newblock Complexity estimates depending on condition and round-off error.
\newblock {\em J. ACM}, 46(1):113--184, 1999.

\bibitem{devore_hanin_petrova_2021}
R.~DeVore, B.~Hanin, and G.~Petrova.
\newblock Neural network approximation.
\newblock {\em Acta Numerica}, 30:327--444, 2021.

\bibitem{Doyle_McMullen}
P.~Doyle and C.~McMullen.
\newblock Solving the quintic by iteration.
\newblock {\em Acta Math.}, 163(3-4):151--180, 1989.

\bibitem{Fawzi_NIPS2018}
A.~Fawzi, H.~Fawzi, and O.~Fawzi.
\newblock Adversarial vulnerability for any classifier.
\newblock In {\em Proceedings of the 32nd International Conference on Neural
  Information Processing Systems}, NIPS'18, pages 1186--1195, Red Hook, NY,
  USA, 2018. Curran Associates Inc.

\bibitem{Deep_Fool1}
A.~Fawzi, S.~M. Moosavi~Dezfooli, and P.~Frossard.
\newblock The robustness of deep networks - a geometric perspective.
\newblock {\em IEEE Signal Processing Magazine}, 34(6):13. 50--62, 2017.

\bibitem{AIM}
C.~Fefferman, A.~C. Hansen, and S.~Jitomirskaya, editors.
\newblock {\em Computational mathematics in computer assisted proofs}, American
  Institute of Mathematics Workshops. American Institute of Mathematics, 2022.
\newblock Available online at
  \url{https://aimath.org/pastworkshops/compproofsvrep.pdf}.

\bibitem{Fefferman_Klartag2}
C.~Fefferman and B.~Klartag.
\newblock {Fitting a $C^m$-smooth function to data II}.
\newblock {\em Revista Matemática Iberoamericana}, 25(1):49 -- 273, 2009.

\bibitem{Fefferman_Klartag}
C.~L. {Fefferman} and B.~{Klartag}.
\newblock {Fitting a \(C^m\)-smooth function to data. I}.
\newblock {\em {Ann. Math. (2)}}, 169(1):315--346, 2009.

\bibitem{finlayson2019adversarial}
S.~G. Finlayson, J.~D. Bowers, J.~Ito, J.~L. Zittrain, A.~L. Beam, and I.~S.
  Kohane.
\newblock Adversarial attacks on medical machine learning.
\newblock {\em Science}, 363(6433):1287--1289, 2019.

\bibitem{gazdag2022generalised}
L.~E. Gazdag and A.~C. Hansen.
\newblock Generalised hardness of approximation and the {SCI} hierarchy -- {On}
  determining the boundaries of training algorithms in {AI}.
\newblock {\em arXiv:2209.06715}, 2022.

\bibitem{Goodfellow-et-al-2016}
I.~Goodfellow, Y.~Bengio, and A.~Courville.
\newblock {\em Deep Learning}.
\newblock MIT Press, 2016.
\newblock \url{http://www.deeplearningbook.org}.

\bibitem{Goodfellow}
I.~Goodfellow, J.~Shlens, and C.~Szegedy.
\newblock Explaining and harnessing adversarial examples.
\newblock In {\em International Conference on Learning Representations}, 2015.

\bibitem{gottschling2020troublesome}
N.~M. Gottschling, V.~Antun, A.~C. Hansen, and B.~Adcock.
\newblock The troublesome kernel: On hallucinations, no free lunches, and the
  accuracy-stability tradeoff in inverse problems.
\newblock {\em SIAM Review}, 67(1):73--104, 2025.

\bibitem{Nina1}
N.~M. Gottschling, P.~Campodonico, V.~Antun, and A.~C. Hansen.
\newblock On the existence of optimal multi-valued decoders and their accuracy
  bounds for undersampled inverse problems.
\newblock {\em arXiv:2311.16898}, 2023.

\bibitem{hamonrobustness}
R.~Hamon, H.~Junklewitz, and I.~Sanchez.
\newblock Robustness and explainability of artificial intelligence - from
  technical to policy solutions.
\newblock {\em Publ. Office European Union}, 2020.

\bibitem{Hansen_JAMS}
A.~C. Hansen.
\newblock On the solvability complexity index, the {$n$}-pseudospectrum and
  approximations of spectra of operators.
\newblock {\em J. Amer. Math. Soc.}, 24(1):81--124, 2011.

\bibitem{Hansen2016ComplexityII}
A.~C. Hansen and O.~Nevanlinna.
\newblock Complexity issues in computing spectra, pseudospectra and resolvents.
\newblock {\em Banach Center Publications}, 112:171--194, 2016.

\bibitem{Hansen2021}
A.~C. Hansen and B.~Roman.
\newblock {\em Structure and Optimisation in Computational Harmonic Analysis:
  On Key Aspects in Sparse Regularisation}, pages 125--172.
\newblock Springer International Publishing, Cham, 2021.

\bibitem{ParReLU}
K.~He, X.~Zhang, S.~Ren, and J.~Sun.
\newblock Delving deep into rectifiers: Surpassing human-level performance on
  imagenet classification.
\newblock In {\em Proceedings of the IEEE international conference on computer
  vision}, pages 1026--1034, 2015.

\bibitem{Heaven_Nature}
D.~{Heaven}.
\newblock {Why deep-learning AIs are so easy to fool}.
\newblock {\em Nature}, 574(7777):163--166, Oct. 2019.

\bibitem{HHdl2019}
C.~F. Higham and D.~J. Higham.
\newblock Deep learning: An introduction for applied mathematicians.
\newblock {\em SIAM Review}, 61:860--891, 2019.

\bibitem{huang2018some}
Y.~Huang et~al.
\newblock Some investigations on robustness of deep learning in limited angle
  tomography.
\newblock In {\em International Conference on Medical Image Computing and
  Computer-Assisted Intervention}, pages 145--153. Springer, 2018.

\bibitem{Madry_BugsNIPS2019}
A.~Ilyas and et. al.
\newblock Adversarial examples are not bugs, they are features.
\newblock In {\em Advances in Neural Information Processing Systems 32: Annual
  Conference on Neural Information Processing Systems 2019, NeurIPS 2019,
  December 8-14, 2019, Vancouver, BC, Canada}, pages 125--136, 2019.

\bibitem{Ko1991ComplexityTO}
K.~Ko.
\newblock {\em Complexity Theory of Real Functions}.
\newblock 1991.

\bibitem{LeCun_Nature}
Y.~LeCun, Y.~Bengio, and G.~Hinton.
\newblock {Deep learning}.
\newblock {\em Nature}, 521(7553):436--444, May 2015.

\bibitem{David1}
Z.~N.~D. Liu and A.~C. Hansen.
\newblock Do stable neural networks exist for classification problems? -- a new
  view on stability in ai.
\newblock {\em arXiv:2401.07874}, 2024.

\bibitem{lovasz1987algorithmic}
L.~Lovasz.
\newblock {\em An Algorithmic Theory of Numbers, Graphs and Convexity}.
\newblock CBMS-NSF Regional Conference Series in Applied Mathematics. Society
  for Industrial and Applied Mathematics, 1987.

\bibitem{LeakyReLU}
A.~L. Maas, A.~Y. Hannun, A.~Y. Ng, et~al.
\newblock Rectifier nonlinearities improve neural network acoustic models.
\newblock In {\em Proc. icml}, volume~30, page~3. Citeseer, 2013.

\bibitem{madry2018towards}
A.~Madry, A.~Makelov, L.~Schmidt, D.~Tsipras, and A.~Vladu.
\newblock Towards deep learning models resistant to adversarial attacks.
\newblock In {\em International Conference on Learning Representations}, 2018.

\bibitem{matiisevich1993hilbert}
Y.~V. Matiyasevich.
\newblock {\em Hilbert's tenth problem}.
\newblock MIT Press, 1993.

\bibitem{Nature_Cancer}
S.~McKinney and et. al.
\newblock International evaluation of an {AI} system for breast cancer
  screening.
\newblock {\em Nature}, 577(7788):89--94, 12/2020 2020.

\bibitem{McMullen1}
C.~McMullen.
\newblock Families of rational maps and iterative root-finding algorithms.
\newblock {\em Ann. of Math. (2)}, 125(3):467--493, 1987.

\bibitem{McMullen2}
C.~McMullen.
\newblock Braiding of the attractor and the failure of iterative algorithms.
\newblock {\em Invent. Math.}, 91(2):259--272, 1988.

\bibitem{Deep_Fool3}
S.~Moosavi-Dezfooli, A.~Fawzi, O.~Fawzi, and P.~Frossard.
\newblock Universal adversarial perturbations.
\newblock In {\em IEEE Conf. on computer vision and pattern recognition}, pages
  86--94, July 2017.

\bibitem{Deep_Fool2}
S.~Moosavi{-}Dezfooli, A.~Fawzi, and P.~Frossard.
\newblock Deepfool: {A} simple and accurate method to fool deep neural
  networks.
\newblock In {\em {CVPR}}, pages 2574--2582. {IEEE} Computer Society, 2016.

\bibitem{Smale_Weinberger}
P.~Niyogi, S.~Smale, and S.~Weinberger.
\newblock A topological view of unsupervised learning from noisy data.
\newblock {\em SIAM J. Comput.}, 40(3):646--663, 2011.

\bibitem{Owhadi_2015}
H.~Owhadi, C.~Scovel, and T.~Sullivan.
\newblock {Brittleness of Bayesian inference under finite information in a
  continuous world}.
\newblock {\em Electronic Journal of Statistics}, 9(1):1 -- 79, 2015.

\bibitem{Owhadi_SIREV}
H.~Owhadi, C.~Scovel, and T.~J. Sullivan.
\newblock {On the brittleness of Bayesian inference}.
\newblock {\em SIAM Rev.}, 57(4):566--582, 2015.

\bibitem{Donoho_PNAS}
V.~Papyan, X.~Y. Han, and D.~L. Donoho.
\newblock Prevalence of neural collapse during the terminal phase of deep
  learning training.
\newblock {\em Proceedings of the National Academy of Sciences},
  117(40):24652--24663, 2020.

\bibitem{pinkus1999approximation}
A.~Pinkus.
\newblock Approximation theory of the {MLP} model in neural networks.
\newblock {\em Acta Numer.}, 8:143--195, 1999.

\bibitem{poonen_2014}
B.~Poonen.
\newblock {\em Undecidable problems: a sampler}, pages 211--241.
\newblock Interpreting G\"odel: Critical Essays. Cambridge University Press,
  2014.

\bibitem{shafahi2018adv}
A.~Shafahi, W.~Huang, C.~Studer, S.~Feizi, and T.~Goldstein.
\newblock Are adversarial examples inevitable?
\newblock {\em International Conference on Learning Representations (ICLR)},
  2019.

\bibitem{MachLearnCUP2014}
S.~Shalev-Shwartz and S.~Ben-David.
\newblock {\em Understanding Machine Learning: From Theory to Algorithms}.
\newblock Cambridge University Press, USA, 2014.

\bibitem{21century_Smale}
S.~Smale.
\newblock Mathematical problems for the next century.
\newblock {\em Mathematical Intelligencer}, 20:7--15, 1998.

\bibitem{smith_2013}
P.~Smith.
\newblock {\em An Introduction to G\"odel's Theorems}.
\newblock Cambridge Introductions to Philosophy. Cambridge University Press, 2
  edition, 2013.

\bibitem{sutton2023adversarial}
O.~J. Sutton, Q.~Zhou, I.~Y. Tyukin, A.~N. Gorban, A.~Bastounis, and D.~J.
  Higham.
\newblock How adversarial attacks can disrupt seemingly stable accurate
  classifiers.
\newblock {\em arXiv preprint arXiv:2309.03665}, 2023.

\bibitem{SzZ-14}
C.~Szegedy, W.~Zaremba, I.~Sutskever, J.~Bruna, D.~Erhan, I.~J. Goodfellow, and
  R.~Fergus.
\newblock Intriguing properties of neural networks.
\newblock In {\em Int. Conf. on Learning Representations}, 2014.

\bibitem{Turing_Machine}
A.~M. Turing.
\newblock On {C}omputable {N}umbers, with an {A}pplication to the
  {E}ntscheidungsproblem.
\newblock {\em Proc. London Math. Soc.}, S2-42(1):230, 1936.

\bibitem{Turing_1950}
A.~M. Turing.
\newblock I.-{C}omputing machinery and intelligence.
\newblock {\em Mind}, LIX(236):433--460, 1950.

\bibitem{tyukin2020adversarial}
I.~Tyukin, D.~Higham, and A.~Gorban.
\newblock On adversarial examples and stealth attacks in artificial
  intelligence systems.
\newblock In {\em 2020 International Joint Conference on Neural Networks
  (IJCNN)}, pages 1--6. IEEE, 2020.

\bibitem{THGW21}
I.~Y. Tyukin, D.~J. Higham, A.~Bastounis, E.~Woldegeorgis, and A.~N. Gorban.
\newblock {The feasibility and inevitability of stealth attacks}.
\newblock {\em IMA Journal of Applied Mathematics}, page hxad027, 10 2023.

\bibitem{wang2023blanchet1}
S.~Wang, N.~Si, J.~Blanchet, and Z.~Zhou.
\newblock On the foundation of distributionally robust reinforcement learning.
\newblock {\em arXiv:2311.09018}, 2023.

\bibitem{Weinberger}
S.~Weinberger.
\newblock {\em Computers, Rigidity, and Moduli: The Large-Scale Fractal
  Geometry of Riemannian Moduli Space}.
\newblock Princeton University Press, USA, 2004.

\bibitem{Johan1}
J.~S. Wind, V.~Antun, and A.~C. Hansen.
\newblock Implicit regularization in ai meets generalized hardness of
  approximation in optimization -- sharp results for diagonal linear networks.
\newblock {\em arXiv:2307.07410}, 2023.

\bibitem{Zakhrevskaya2001}
N.~S. Zakrevskaya and A.~P. Kovalevskii.
\newblock One-parameter probabilistic models of text statistics.
\newblock {\em Sib. Zh. Ind. Mat.}, 4:142--153, 2001.

\end{thebibliography}

\end{document}